\documentclass[a4paper]{article}
\usepackage[utf8x]{inputenc}

\usepackage[a4paper,top=3cm,bottom=2cm,left=3cm,right=3cm,marginparwidth=1.75cm]{geometry}
\usepackage[numbers,sort]{natbib}
\usepackage{amsmath,amsfonts,amsthm}
\usepackage{xspace}
\usepackage{algorithm}
\usepackage{algorithmic}
\usepackage{graphicx}
\usepackage[colorlinks=true, allcolors=blue]{hyperref}
\usepackage{booktabs}
\usepackage{multirow}
\usepackage{subcaption}
\usepackage{pgfplots}
\usepackage{filecontents}
\usepackage{tikz}
\usetikzlibrary{calc}
\usetikzlibrary{mindmap}
\usetikzlibrary{arrows}
\usetikzlibrary{shadows}
\usetikzlibrary{plotmarks}
\usetikzlibrary{backgrounds}
\usetikzlibrary{shapes}
\usetikzlibrary{shapes.symbols}
\newcommand{\onemax}{\textsc{OneMax}\xspace}

\newcommand{\lotz}{\textsc{LOTZ}\xspace}
\newcommand{\leadingones}{\textsc{LeadingOnes}\xspace}
\newcommand{\omm}{\textsc{OneMinMax}\xspace}

\newcommand{\harm}{\mathrm{H}_n}

\newcommand{\ones}[1]{\left|#1\right|_1}
\newcommand{\zeros}[1]{\left|#1\right|_0}

\newcommand{\hvc}{\textsc{HVC}\xspace}
\newcommand{\HVC}[2]{\hvc\mathord{\left(#1,#2\right)}}
\newcommand{\cdc}{\textsc{CDC}\xspace}
\newcommand{\CDC}[2]{\cdc\mathord{\left(#1,#2\right)}}
\newcommand{\ih}{\mathrm{I_H}}
\newcommand{\Ih}[1]{\ih\mathord{\left(#1\right)}}
\newcommand{\Score}{\textsc{C}\xspace}
\newcommand{\score}[2]{\ensuremath{\Score\left(#1, #2\right)}}

\newcommand{\mrow}[3]{\multirow{#1}{#2}{#3}}
\newcommand{\expnumber}[2]{{#1}\mathrm{E+}{#2}}

\newcommand{\property}{diversity-favouring\xspace}
\newcommand{\good}{good\xspace}
\newcommand{\bad}{bad\xspace}
\newcommand{\pgood}{p_{\mathrm{\good}}}
\newcommand{\front}{F^*}
\newcommand{\paretoset}{X^*}
\newcommand{\ld}{\mathrm{L}}
\newcommand{\LD}[1]{\ld\mathord{\left(#1\right)}}
\newcommand{\lo}{\mathrm{LO}}
\newcommand{\LO}[1]{\lo\mathord{\left(#1\right)}}
\newcommand{\tz}{\mathrm{TZ}}
\newcommand{\TZ}[1]{\tz\mathord{\left(#1\right)}}
\newcommand{\f}[2]{f_#1(#2)}

\newcommand{\nsgaii}{\textsc{NSGA-II}\xspace}
\newcommand{\speatwo}{\textsc{SPEA2}\xspace}
\newcommand{\ibea}{\textsc{IBEA}\xspace}
\newcommand{\smsemoa}{\textsc{SMS-EMOA}\xspace}

\newcommand{\prob}{\mathrm{Prob}}
\newcommand{\Prob}[1]{\prob\mathord{\left(#1\right)}}

\newcommand{\ie}{i.\,e.,\xspace}
\newcommand{\eg}{e.\,g.,\xspace}

\newcommand{\ignore}[1]{}

\newtheorem{theorem}{Theorem}[section]
\newtheorem{lemma}[theorem]{Lemma}

\newtheorem{definition}[theorem]{Definition}

\title{Design and Analysis of Diversity-Based Parent~Selection Schemes for Speeding Up Evolutionary~Multi-objective Optimisation}

\author{Edgar Covantes Osuna\footnote{Department of Computer Science, The University of Sheffield, Sheffield, United Kingdom.} \and Wanru Gao\footnote{Optimisation and Logistics, School of Computer Science, The University of Adelaide, Adelaide, Australia.} \and Frank Neumann\footnotemark[2] \and Dirk Sudholt\footnotemark[1]}

\begin{document}
\maketitle
\begin{abstract}
Parent selection in evolutionary algorithms for multi-objective optimisation is usually performed by dominance mechanisms or indicator functions that prefer non-dominated points.
We propose to refine the parent selection on evolutionary multi-objective optimisation with diversity-based metrics. The aim is to focus on individuals with a high diversity contribution located in poorly explored areas of the search space, so the chances of creating new non-dominated individuals are better than in highly populated areas. We show by means of rigorous runtime analysis that the use of diversity-based parent selection mechanisms in the Simple Evolutionary Multi-objective Optimiser (SEMO) and Global SEMO for the well known bi-objective functions \omm and \lotz can significantly improve their performance.
Our theoretical results are accompanied by experimental studies that show a correspondence between theory and empirical results and motivate further theoretical investigations in terms of stagnation. We show that stagnation might occur when favouring individuals with a high diversity contribution in the parent selection step and provide a discussion on which scheme to use for more complex problems based on our theoretical and experimental results.
\end{abstract}

\section{Introduction}
\label{sec:intro}

Evolutionary algorithms have been used for a wide range of complex optimisation and design problems in various areas such as engineering, logistics, and art. Selection plays a crucial role in the use of evolutionary algorithms as it sets the direction of the evolutionary process. An evolutionary algorithm consists of two parts where selection of individuals is carried out. Parent selection decides on which individuals of the current population produce offspring, whereas survival selection selects the population for the next generation from the current set of parents and offspring after the offspring population has been produced.

The area of evolutionary multi-objective optimisation (EMO) designs pop\-u\-la\-tion-based evolutionary algorithms (EAs) where the population is used to approximate the so-called Pareto front. Given that EAs use a population which is a set of solutions to a given problem, EAs are suited in a natural way for computing trade-offs with respect to two (or more) conflicting objective functions.

Well established multi-objective evolutionary algorithms (MOEAs) such as \nsgaii \cite{Deb2002}, \speatwo \cite{Bleuler2001}, \ibea \cite{Zitzler2004} have two basic principles driven by selection. First of all, the goal is to push the current population close to the ``true'' Pareto front. The second goal is to ``spread'' the population along the front such that it is well covered. The first goal is usually achieved by dominance mechanisms between the search points or indicator functions that prefer non-dominated points. The second goal involves the use of diversity mechanisms. Alternatively, indicators such as the hypervolume indicator play a crucial role to obtain a good spread of the different solutions of the population along the Pareto front.

In the context of EMO, parent selection is often uniform whereas survival selection is based on dominance and the contribution of an individual to the diversity of the population. In this paper, we explore the use of different parent selection mechanisms in EMO. The goal is to speed up the optimisation process of an EMO algorithm by selecting individuals that have a high chance of producing beneficial offspring. To our knowledge the use of different parent selection schemes has not been widely studied and there are only a few algorithms placing emphasis on selecting good parents for reproduction.
\nsgaii~\cite{Deb2002} and \speatwo~\cite{Bleuler2001} focus on survival selection. However, both use tournament selection based on Pareto ranking and their incorporated diversity measure to select the parents. We establish a similar ranking of the individuals in the parent population and examine a wide range of parent selection distributions and their impact on the performance of our studied algorithms.
In~\cite{Phan2011} a MOEA with parent selection using a so-called prospect indicator is used to improve \smsemoa. The prospect indicator evaluates the potential (or prospect) of an individual to reproduce offspring that dominates itself. Their experimental results show improvement over classical MOEAs.

The parent selection mechanisms studied in this paper use the diversity contribution of an individual in the parent population to select promising individuals for reproduction. 
The main assumption is that individuals with a high diversity score are located in poorly explored or less dense areas of the search space, so the chances of creating new non-dominated individuals are better than in areas where there are several individuals. In this sense we have designed parent selection schemes for MOEAs that let the MOEA focus on individuals where the neighbourhood is not fully covered and in consequence, force the reproduction in those areas and to the spread of the population along the search space. 

In our investigations, we focus on parent selection mechanisms that favour individuals having a high hypervolume contribution (HVC) or high crowding distance contribution (CDC). HVC plays a crucial role in the survival selection of hypervolume-based EMO algorithms whereas the crowding distance measure is used in popular algorithms such as \nsgaii.
We propose several different parent selection mechanisms that take one of these two measures and then select individuals according to their diversity contribution. The different selection mechanisms differ in their selection strength, from mild preferences for more appealing parents to more aggressive schemes that yield a quite drastic change of behaviour.
Specifically, we propose schemes based on the ranks of the individuals according to their diversity contribution, selecting according to an exponential, power law, or harmonic distribution. Furthermore, we consider tournament selection, selecting the individuals with the highest diversity contribution (HDC) as well as a ranking scheme called Non-Minimum Uniform at Random (NMUAR) which ignores the individuals with the minimum diversity contribution.

We show by means of rigorous runtime analysis that the use of diversity-based parent selection mechanisms can significantly improve the performance of MOEAs. The area of runtime analysis has contributed significantly to the theoretical understanding of EMO algorithms~\cite{Giel2010,Friedrich2011,Horoba2010,Qian2016} and allows to study different components of EMO methods from a rigorous perspective.
In order to gain insights into the potential benefits of the diversity-based parent selection mechanisms, we study the functions \omm and \lotz (Leading Ones, Trailing Zeroes) introduced in~\cite{Giel2010} and \cite{Laumanns2004}, respectively. \omm generalizes the well-known \onemax function and \lotz generalizes the well-known \leadingones problem to the multi-objective case. Both functions have been examined in a wide range of theoretical studies for variants of the SEMO algorithm. Other studies in the area of runtime analysis of MOEAs consider hypervolume-based algorithms~\cite{Nguyen2015,Doerr2016}, namely a variant of \ibea, and MOEAs incorporating other diversity mechanisms for survival selection~\cite{Horoba2010}.

We show that the use of various diversity-based parent selection mechanisms speeds up SEMO by factors of order~$n$ or $n/\log n$ for \omm and \lotz with regards to the expected time for finding the whole Pareto front. For \lotz the use of rank-based parent selection can reduce the expected time to compute the whole Pareto front from $\Theta(n^3)$ to $O(n^2)$ (see \cite{Motwani1995} for the asymptotic notation). Studying \omm, we show a similar effect, \ie that the expected time reduces from $\Theta(n^2 \log n)$  to $O(n \log n)$ for our best performing rank-based parent selection methods. The results for \omm also hold for Global SEMO (GSEMO) which uses standard bit mutations where every bit in the mutation step is flipped with probability $1/n$.

This article extends its conference version~\cite{Covantes2017} in various ways. In~\cite{Covantes2017} for \lotz only SEMO was analysed as the analysis of GSEMO was too challenging. Here we address this challenge by providing investigations for a variant of GSEMO on \lotz. This modified GSEMO uses a feature we call $\ld$-dominant attribute, which ensures that individuals closest to the front are selected in the parent selection step. Furthermore, we provide additional experimental results. This includes a detailed experimental investigation on the stagnation probabilities for parent selection methods that are in some cases not able to obtain the whole Pareto front. These experimental results motivate new additional theoretical analyses of the stagnation probability for very greedy schemes for GSEMO with the $\ld$-dominant attribute on \lotz as well as SEMO and GSEMO on \omm provided in Section~\ref{sec:discgreed}.
We point out situations for \lotz where using parent selection to focus on the highest diversity contribution can lead to stagnation if global mutations are being used. However, the same parent selection mechanism is effective for SEMO where only local mutations are being used. Investigating \omm and NMUAR in the parent selection step, we show that the choice of the reference point for hypervolume-based selection can make the difference between stagnation and an expected polynomial time. Namely, we show that choosing the reference point as $(-n-1,-1)$ for NMUAR has a positive probability of reaching stagnation whereas any symmetric reference point $(-r,-r)$, $r \geq 1$, leads to an expected time of $O(n^2)$.
Finally, we discuss our findings and conclude that the use of a power-law distribution within the parent selection provides the best trade-off between speed and the risk of stagnation.

The outline of the paper is as follows. In Section~\ref{sec:pre}, we introduce the algorithms and problems that are subject to our investigations. Section~\ref{sec:divparsec} establishes the algorithmic framework used in the theoretical and experimental analysis. Section~\ref{sec:progress} establishes some general properties that enable speed-ups through diversity-based parent selection. Our rigorous runtime results for \omm and \lotz are presented in Section~\ref{sec:ommpro} and~\ref{sec:lotzpro}, respectively. An experimental study complementing the theoretical results is presented in Section~\ref{sec:exp} and additional experimentally motivated theoretical studies on the effectiveness of greediness in parent selection are presented in Section~\ref{sec:discgreed}.
Finally, we finish with some discussion and concluding remarks.

\section{Preliminaries}
\label{sec:pre}

In our investigations we consider problems ${f=(f_1, \ldots, f_m)\colon \{0,1\}^n \rightarrow \mathbb{R}^m}$. Throughout this paper, we assume without loss of generality that each function $f_i$, $1 \leq i \leq m$, should be maximised. As there is no single point that maximises all functions simultaneously, the goal is to find a set of so-called Pareto-optimal solutions.

\begin{definition}[Pareto optimality]
\label{def:paropt}
Let~$f:X \to F$, where~$X\subseteq\{0,1\}^n$ is called decision space and~$F\subseteq \mathbb{R}^m$ objective space. The elements of~$X$ are called decision vectors and the elements of~$F$ objective vectors. A decision vector~$x \in X$ is Pareto optimal if there is no other~$y \in X$ that dominates~$x$. $y$ dominates~$x$, denoted as~$y \succ x$, if~$f_i(y) \geq f_i(x)$ for all~$i=1,\ldots,m$ and~$f_i(y)>f_i(x)$ for at least one index~$i$. A decision vector~$y$ weakly dominates~$x$, denoted by $y\succeq x$, if $f_i(y) \geq f_i(x)$, for all $i$. The set of all Pareto-optimal decision vectors~$\paretoset$ is called Pareto set. $\front=f(\paretoset)$ is the set of all Pareto-optimal objective vectors and denoted as Pareto front.
\end{definition}

We consider \omm and \lotz (see Definition~\ref{def:omm} and \ref{def:lotz}) which are benchmark functions that facilitate the theoretical analysis. These functions have previously been used in the theoretical analysis of evolutionary algorithms and our choice therefore allows for comparisons with previous approaches such as the ones investigated in~\cite{Giel2003,Laumanns2004,Giel2010}. 

\begin{definition}[\omm]
\label{def:omm}
A pseudo-Boolean function $\{0,1\}^n\to \mathbb{N}^2$ with the objective functions
\begin{displaymath}
\omm(x_1,\ldots,x_n):= \left(\sum_{i=1}^{n}x_i,n-\sum_{i=1}^{n}x_i\right),
\end{displaymath}
where the aim is to maximise the number of ones and zeroes at the same time (see Figure~\ref{fig:omm}).
\end{definition}

\begin{definition}[Leading Ones, Trailing Zeroes, \lotz]
\label{def:lotz}
A pseudo-Boolean function $\{0,1\}^n\to \mathbb{N}^2$ defined as
\begin{displaymath}
\lotz(x_1,\ldots,x_n):=\left(\sum_{i=1}^{n}\prod_{j=1}^{i}x_j,\sum_{i=1}^{n}\prod_{j=i}^{n}(1-x_j) \right),
\end{displaymath}
where the goal is to simultaneously maximise the number of leading ones and trailing zeroes (see Figure~\ref{fig:lotz}).
\end{definition}

\omm has the property that every single solution represents a point in the Pareto front and that no search point is strictly dominated by another one. The goal is to cover the whole Pareto front, \ie to compute a set of individuals that contains for each $i$, $0 \le i \le n$, an individual with exactly $i$ ones. In the case of \lotz, all non-Pareto optimal decision vectors only have Hamming neighbours that are better or worse, but never incomparable to it. This fact facilitates the analysis of the population-based algorithms, which certainly cannot be expected from other multi-objective optimisation problems. Note that the Pareto front for \lotz is given by the set of $n+1$ search points $\{1^i 0^{n-i} \mid 0 \le i \le n\}$.

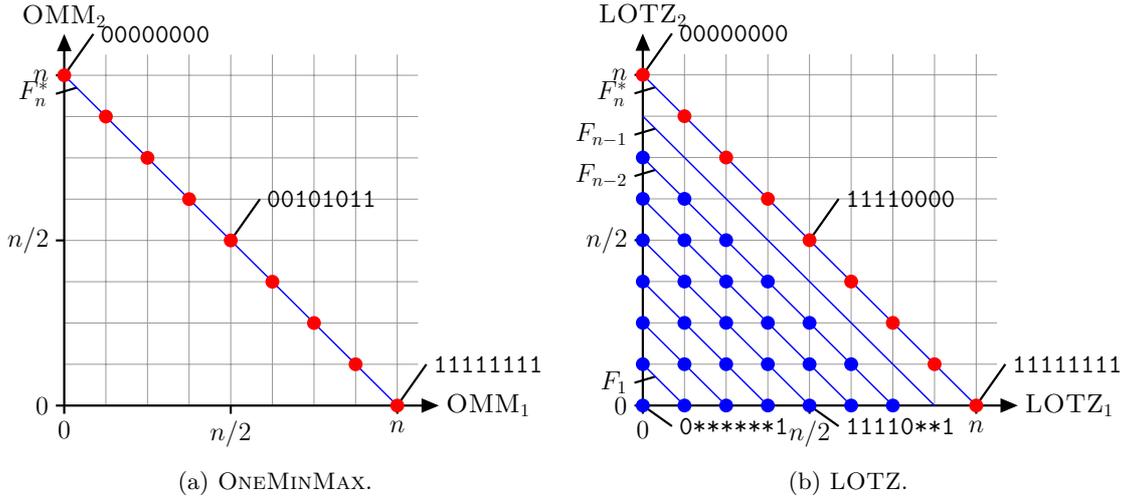
\begin{figure*}[t]
    \centering
    \begin{subfigure}[t]{0.49\textwidth}
        \centering
     \resizebox{\linewidth}{!}{
		\begin{tikzpicture}[domain=0:8,xscale=0.8,yscale=0.8,scale=1, every shadow/.style={shadow xshift=0.0mm, shadow yshift=0.4mm}]
          \tikzstyle{helpline}=[black,very thick];
          \tikzstyle{function}=[blue,thick];
          \tikzstyle{individual}=[blue,very thick];
          \tikzstyle{pareto}=[red,very thick];
		  \draw[black!40,line width=0.2pt,xstep=1,ystep=1] (0,0) grid (8.5,8.5);          
          \draw[helpline, -triangle 45] (0,0) -- (0,9) node[above] {\Large $\textsc{OMM}_{2}$}; 
          \draw[helpline, -triangle 45] (0,0) -- (9,0) node[right] {\Large $\textsc{OMM}_{1}$};
          
          \draw[helpline] (0.3,7.7) -- (-0.2,7.55) node[left] {\Large $F^*_n$};      
          \draw[helpline] (8,0) -- (8.7,1) node[right] {\Large {\tt 11111111}};
          \draw[helpline] (0,8) -- (0.7,9) node[right] {\Large {\tt 00000000}};
          \draw[helpline] (4,4) -- (4.7,5) node[right] {\Large {\tt 00101011}};
          
          \draw[helpline] (0,0) -- (-0.2,0) node[left] {\Large $0$};
          \draw[helpline] (0,4) -- (-0.2,4) node[left] {\Large $n/2$};
          \draw[helpline] (0,8) -- (-0.2,8) node[left] {\Large $n$};
          \draw[helpline] (0,0) -- (0,-0.2) node[below] {\Large $0$};
          \draw[helpline] (4,0) -- (4,-0.2) node[below] {\Large $n/2$};
          \draw[helpline] (8,0) -- (8,-0.2) node[below] {\Large $n$};
          \draw[function] (0,8) -- (8,0);
		  \foreach \x/\y in {0/8,1/7,2/6,3/5,4/4,5/3,6/2,7/1,8/0}
                \filldraw[pareto] (\x,\y) circle (4pt);              
		\end{tikzpicture}
        }
        \caption{\omm.}
		\label{fig:omm}
    \end{subfigure}
    \begin{subfigure}[t]{0.50\textwidth}
        \centering
    \resizebox{\linewidth}{!}{
		\begin{tikzpicture}[domain=0:8,xscale=0.8,yscale=0.8,scale=1, every shadow/.style={shadow xshift=0.0mm, shadow yshift=0.4mm}]
          \tikzstyle{helpline}=[black,very thick];
          \tikzstyle{function}=[blue,thick];
          \tikzstyle{individual}=[blue,very thick];
          \tikzstyle{pareto}=[red,very thick];
		  \draw[black!40,line width=0.2pt,xstep=1,ystep=1] (0,0) grid (8.5,8.5);          
          \draw[helpline, -triangle 45] (0,0) -- (0,9) node[above] {\Large $\lotz_{2}$}; 
          \draw[helpline, -triangle 45] (0,0) -- (9,0) node[right] {\Large $\lotz_{1}$};
          
          \draw[helpline] (0,0) -- (-0.2,0) node[left] {\Large $0$};
          \draw[helpline] (0,4) -- (-0.2,4) node[left] {\Large $n/2$};
          \draw[helpline] (0,8) -- (-0.2,8) node[left] {\Large $n$};
          \draw[helpline] (0,0) -- (0,-0.2) node[below] {\Large $0$};
          \draw[helpline] (4,0) -- (4,-0.2) node[below] {\Large $n/2$};
          \draw[helpline] (8,0) -- (8,-0.2) node[below] {\Large $n$};
          
          \draw[helpline] (0.3,0.7) -- (-0.2,0.55) node[left] {\Large $F_1$};    
		  \draw[helpline] (0.3,5.7) -- (-0.2,5.55) node[left] {\Large $F_{n-2}$};
		  \draw[helpline] (0.3,6.7) -- (-0.2,6.55) node[left] {\Large $F_{n-1}$};
		  \draw[helpline] (0.3,7.7) -- (-0.2,7.55) node[left] {\Large $F^*_n$};      
		  \draw[helpline] (0,0) -- (0.7,-0.5) node[right] {\Large {\tt 0******1}};          
          \draw[helpline] (4,0) -- (4.7,-0.5) node[right] {\Large {\tt 11110**1}};
          
          \draw[helpline] (8,0) -- (8.7,1) node[right] {\Large {\tt 11111111}};
          \draw[helpline] (0,8) -- (0.7,9) node[right] {\Large {\tt 00000000}};
          \draw[helpline] (4,4) -- (4.7,5) node[right] {\Large {\tt 11110000}};
          
          \foreach \x in {1,...,8}
          	\draw[function] (0,\x) -- (\x,0);
		  \filldraw[individual] (0,0) circle (4pt);
		  \foreach \x/\y in {0/1,1/0}
                \filldraw[individual] (\x,\y) circle (4pt);
		  \foreach \x/\y in {0/2,1/1,2/0}
                \filldraw[individual] (\x,\y) circle (4pt);	  
		  \foreach \x/\y in {0/3,1/2,2/1,3/0}
                \filldraw[individual] (\x,\y) circle (4pt);	
          \foreach \x/\y in {0/4,1/3,2/2,3/1,4/0}
                \filldraw[individual] (\x,\y) circle (4pt);
          \foreach \x/\y in {0/5,1/4,2/3,3/2,4/1,5/0}
                \filldraw[individual] (\x,\y) circle (4pt);
          \foreach \x/\y in {0/6,1/5,2/4,3/3,4/2,5/1,6/0}
                \filldraw[individual] (\x,\y) circle (4pt);	      
		  \foreach \x/\y in {0/8,1/7,2/6,3/5,4/4,5/3,6/2,7/1,8/0}
                \filldraw[pareto] (\x,\y) circle (4pt);              
		  
		\end{tikzpicture}
        }
        \caption{\lotz.}
		\label{fig:lotz}
    \end{subfigure}
    \caption{Sketches of the functions \omm (\textsc{OMM}) and \lotz with $n=8$.}
    \label{fig:functions}
\end{figure*}

We focus our analysis on two simple MOEAs, SEMO and its variant called Global SEMO (GSEMO) because of their simplicity and suitability for a rigorous theoretical analysis. SEMO starts with an initial solution $s\in\{0,1\}^n$ chosen uniformly at random. All non-dominated solutions are stored in the population~$P$. Then, it selects a solution $s$ uniformly at random from $P$, and a new search point $s'$ is produced by the mutation step which flips one bit of $s$ chosen uniformly at random. The new population contains for each non-dominated fitness vector $f(s)$, $s\in P \cup \{s'\}$, one corresponding search point (dominated individuals are removed from the population), and in the case where $f(s')$ is not dominated, $s'$ is added to $P$ (see Algorithm~\ref{alg:semo}).

\begin{algorithm}[!ht]
  \begin{algorithmic}[1]
  	\STATE Choose an initial solution $s\in\{0,1\}^n$ uniformly at random.
    \STATE Determine $f(s)$ and initialize $P:=\{s\}$.
    \WHILE{stopping criterion \NOT met}
    	\STATE{Choose $s$ uniformly at random from $P$.}
        \STATE{Choose $i \in \{1,\ldots,n\}$ uniformly at random.}
        \STATE{Define $s'$ by flipping the $i$-th bit of $s$.}
        \IF{$s'$ is \NOT dominated by any individual in $P$} 
        	\STATE{Add $s'$ to $P$, and remove all individuals weakly dominated by $s'$ from~$P$.}
        \ENDIF
    \ENDWHILE
  \end{algorithmic}
  \caption{SEMO}
  \label{alg:semo}
\end{algorithm}

For SEMO, we know that the expected running time on \omm is at most $O(n^2\log n)$ \cite{Giel2010}. We prove that this upper bound is asymptotically tight.

\begin{theorem}
\label{the:semo_oneminmax}
The expected time for SEMO to cover the whole Pareto front on \omm is $\Theta(n^2 \log n)$.
\end{theorem}
\begin{proof}
The upper bound was shown in~\cite{Giel2010}. For the lower bound, let $\ones{x}$ denote the number of 1-bits and $\zeros{x}$ denote the number of 0-bits in $x$. Define $X_t := \min_{x \in P_t}\{\ones{x}\}$ if for the initial search point $x_0$ we have $\ones{x_0} \ge n/2$, and $X_t := \min_{x \in P_t}\{\zeros{x}\}$ otherwise. Note that, by definition, $X_0 \ge n/2$. Now, $X_t = 0$ is a necessary requirement for covering the whole Pareto front at time~$t$. Hence we lower-bound the sought time by the expected time for $X_t$ to reach value~0.

Since only local mutations are used, $X_t$ can only decrease by~1. In order to decrease $X_t$ we have to select a parent with Hamming distance $X_t$ to $0^n$ or $1^n$, respectively, which happens with probability~$1/|P_t|$. Note that $|P_t| \ge n/2 - X_t$ as the population contains individuals with $X_t, X_t + 1, \dots, \lceil n/2\rceil$ ones. Moreover, mutation needs to flip one of the $X_t$ bits differing to $0^n$ or $1^n$, respectively. Hence
\[
\Prob{X_{t+1} = X_t - 1 \mid X_t} \le \frac{1}{n/2 - X_t} \cdot \frac{X_t}{n}.
\]
The total expected time to decrease $X_t$ to~0 is thus at least
\[
\sum_{j=1}^{n/2} \left(\frac{n}{2}-j\right) \frac{n}{j} =
\sum_{j=1}^{n/2} \frac{n^2}{2j} - \sum_{j=1}^{n/2} n
= \frac{n^2 \ln n}{2} - O(n^2)
\]
as $\sum_{j=1}^{n/2} 1/j \ge \ln n/2 = \ln n - \ln 2$.
\end{proof}

The reason for the relatively high running time is that the growing population slows down exploration. The population can only expand on the Pareto front in case search points with the current highest or lowest number of ones are chosen (corresponding to a minimum $X_t$-value in the proof of Theorem~\ref{the:semo_oneminmax}). Once the population has grown to a size of $\mu = \Theta(n)$, the probability that this happens has decreased to $\Theta(1/n)$. This means that only a $\sim 1/n$-th fraction of the time the algorithm has a chance to expand on the Pareto front! Uniform parent selection means that most steps are spent idling. The same effect occurs for SEMO on \lotz as proved in~\cite{Laumanns2004}.

\begin{theorem}[Lemma~2 in~\cite{Laumanns2004}]
\label{the:semo_lotz}
The expected time for SEMO to cover the whole Pareto front on \lotz is $\Theta(n^3)$.
\end{theorem}

In the case of GSEMO, a new solution $s'$ is created by flipping each bit from a solution $s$ independently with probability~$1/n$, then it proceeds in the same way as SEMO (see Algorithm \ref{alg:gsemo}). For GSEMO we have upper bounds of the same order, $O(n^2 \log n)$ for \omm~\cite{Giel2010} and $O(n^3)$ for \lotz~\cite{Giel2003}, though no lower matching bound is available in the literature for the case of GSEMO on \lotz. 

\begin{algorithm}[!ht]
  \begin{algorithmic}[1]
  	\STATE Choose an initial solution $s\in\{0,1\}^n$ uniformly at random.
    \STATE Determine $f(s)$ and initialize $P:=\{s\}$.
    \WHILE{stopping criterion \NOT met}
    	\STATE{Choose $s$ uniformly at random from $P$.}
        \STATE{Define $s'$ by flipping each bit in $s$ independently with probability $1/n$.}
        \IF{$s'$ is \NOT dominated by any individual in $P$} 
        	\STATE{Add $s'$ to $P$, and remove all individuals weakly dominated by $s'$ from~$P$.}
        \ENDIF
    \ENDWHILE
  \end{algorithmic}
  \caption{GSEMO}
  \label{alg:gsemo}
\end{algorithm}

We remark that \lotz can also be optimised more efficiently, in time $O(n^2)$, by a tailored algorithm that uses local search along individual objectives during initialisation to locate both extreme points of the Pareto front, $0^n$ and $1^n$, and then uses crossover to produce the whole Pareto front from these points~\cite{Qian2013}. Incorporating a fairness mechanism which makes sure that each individual produces roughly the same number of offspring into SEMO leads to the algorithm FEMO. For FEMO a runtime bound of $\Theta(n^2 \log n)$ has been given in~\cite{Laumanns2004}. The runtime analysis provided for \ibea in~\cite{Nguyen2015} gives an upper bound of $O(n^2 \log n)$ and $O(n^3)$ for \omm and \lotz, respectively, if the population size is set to $n+1$ and therefore does not improve on the results for SEMO given in~\cite{Laumanns2004}.

Our aim is to develop rigorous runtime bounds of SEMO and GSEMO introducing different diversity-based parent selection. We want to study how these mechanisms help to improve the performance of the MOEAs.

\section{Diversity-Based Parent Selection}
\label{sec:divparsec}
Hypervolume-based EAs have become very popular in recent years for multi-objective optimisation where the hypervolume indicator is used as a measurement of the coverage of the population~\cite{Auger2012,Zitzler2004}. The hypervolume indicator measures a set of elements corresponding to images of the individuals with the volume of the dominated portion of the objective space. It is calculated based on the selection of a reference point. In particular, given a reference point $r \in \mathbb{R}^m$, the hypervolume indicator is defined on a set $P \subset S$ as
\[
\Ih{P} = \lambda\left( \bigcup_{x\in P} [\f{1}{x},r_1] \times [\f{2}{x},r_2] \times \cdots \times [\f{m}{x},r_m]\right)
\]
where $\lambda(S)$ denotes the Lebesgue measure of a set $S$ and $[\f{1}{a},r_1] \times [\f{2}{a},r_2] \times \cdots \times [\f{m}{a},r_m]$ is the orthotope with $f(a)$ and $r$ in opposite corners. We define the contribution of an element $x \in P$ to the hypervolume of a set of elements $P$ as
\[
\HVC{x}{P} = \Ih{P} - \Ih{P \setminus \{x\}}.
\]

The calculation of hypervolume indicator and the calculation of the contribution are both NP-hard when the number of objectives $m$ is a parameter \cite{Bringmann2010,Bringmann2012}. However, both can be computed in polynomial time if $m$ is fixed. In the following, for bi-objective problems like \omm and \lotz, we can directly calculate the contribution of an element by taking into account the two direct neighbours in the objective space as follows.

\begin{definition}[Hypervolume contribution]
\label{def:hypcon}
For a given reference point $r = (r_1, r_2)$, we set $\f{1}{x_0}=r_1$ and $\f{2}{x_{\mu+1}}=r_2$ where $x_0$ and $x_{\mu+1}$ are individuals used to estimate the hypervolume contribution, and hereinafter $\mu$ denotes the size of the current population in SEMO/GSEMO. Furthermore, we assume that $r_1=\f{1}{x_0} < \f{1}{x_1}$, $r_2 = \f{2}{x_{\mu+1}} < \f{2}{x_{\mu}}$.

Let the population be sorted according to the value of $\f{1}{x_i}$ such that
\[
\f{1}{x_0}<\f{1}{x_1}<\f{1}{x_2} < \cdots < \f{1}{x_\mu}.
\]
The contribution of an individual $x_i$ to the hypervolume of a population $P$ is then given by
\[
\HVC{x_i}{P}=(\f{1}{x_i}-\f{1}{x_{i-1}})\cdot(\f{2}{x_{i}}-\f{2}{x_{i+1}}).
\]
\end{definition}

Another diversity metric applied to our framework is the \emph{crowding distance} used in \nsgaii~\cite{Deb2002}. The crowding distance operator measures the density of solutions surrounding a particular solution in the population. A solution with a lower crowding distance value implies that the region occupied by this solution is crowded by other solutions. The solutions with a higher crowding distance value are chosen/preferred for reproduction.

Since both SEMO and GSEMO use a population of non-dominated individuals, \ie all individual have the minimum non-domination rank possible, we can directly apply the crowding distance as our diversity metric (Algorithm~\ref{alg:crowdist}). The population is sorted for each objective function value in increasing order of magnitude. Thereafter, for each objective function, the boundary solutions (solutions with smallest and largest function values) are assigned an infinite distance value. All other intermediate solutions are assigned a distance value equal to the absolute normalised difference of the function values of two adjacent solutions (see Line~\ref{alg:crowdist:dist} of Algorithm~\ref{alg:crowdist}, $f_{m}^{\max}$ and $f_{m}^{\min}$ are the maximum and minimum values of the $m$-th objective function).

\begin{algorithm}[!ht]
  \begin{algorithmic}[1]
   	\STATE Let $l:=|P|$.
    \FORALL{$i$ individuals $\in P$}
    	\STATE Set $P[i]_{\mathrm{.distance}}:=0$
    \ENDFOR
    \FORALL{$m$ objectives}
    	\STATE Sort $P$ according to $m$ objective function value in ascending order.
		\STATE $P[1]_{\mathrm{.distance}}:=P[l]_{\mathrm{.distance}}:=\infty$.
        \FOR{ $i = 2$ \TO $l-1$ }
        	\STATE $P[i]_{\mathrm{.distance}}:=P[i]_{\mathrm{.distance}}+\frac{P[i+1]_{.m}-P[i-1]_{.m}}{f_{m}^{\max}-f_{m}^{\min}}$ \label{alg:crowdist:dist}
        \ENDFOR
    \ENDFOR
  \end{algorithmic}
  \caption{Crowding Distance Operator}
  \label{alg:crowdist}
\end{algorithm}

As in previous theoretical studies, we measure the running time as the number of function evaluations needed to fully cover the Pareto front. This common practice is motivated by the fact that function evaluations are often the most time-consuming operations. Note that for SEMO and GSEMO the number of function evaluations coincides with the number of generations needed as each generation only creates one new offspring whose fitness is evaluated.

For the hypervolume contribution (HVC), according to Definition~\ref{def:hypcon}, the reference point can be defined so that the current extreme individuals in the population and individuals in intermediate empty areas have a high diversity score, and a strong influence for the algorithm. In the case of the crowding distance contribution (CDC) the same behaviour applies, extreme points in the search space receive a high distance while intermediate individuals surrounded by empty areas receive a higher distance than the ones where the area is more crowded.

With this information we can define selection mechanisms capable of selecting those extreme points and pushing the spread of the population toward the outer areas of the search space. However, as our theoretical analysis will show, in case the population already contains the extreme points of the Pareto front ($0^n$ and $1^n$ for \omm and \lotz), we need to be flexible enough to ignore those points and select intermediate individuals surrounded by empty areas in the search space to fully cover the Pareto front.

The selection mechanisms defined in this paper use the previous diversity contribution metrics but any other metric can be easily applied that follows the behaviour mentioned before. Firstly, we define 3 different rank-based selection schemes in which the probability of selecting individuals with a high diversity score is higher than for individuals with a lower diversity score (see Definition~\ref{def:selpre}). The first is called \emph{exponential}; it is a rather aggressive scheme that strongly favours the best-ranked individuals and has a very small tail. The second is called \emph{power law} as it follows a power law distribution; it is much less aggressive with a fat tail and yet a constant probability of selecting the first constant ranks. And finally, the third ranking scheme is called \emph{harmonic}; it is the least aggressive scheme with a fat tail and only a probability of $O(1/(\log \mu))$ for selecting the best few individuals.

\begin{definition}[Rank-based selection schemes]
\label{def:selpre}
The probability of selecting the $i$-th ranked individual is
\begin{alignat*}{3}
 \frac{2^{-i}}{\displaystyle \sum_{j=1}^{\mu} 2^{-j}}, &\quad \frac{1/i^2}{\displaystyle \sum_{j=1}^{\mu} \frac{1}{j^2}}, &\quad \frac{1/i}{\displaystyle \sum_{j=1}^{\mu} \frac{1}{j}}
\end{alignat*}
for the exponential, power law, and harmonic ranking scheme (see Figure~\ref{fig:probtails}), 
respectively.
\begin{figure}[!ht]
	\centering
		\begin{tikzpicture}
			\begin{axis}[ymin=0,
            	ymax=0.65,
            	grid=both,
            	enlargelimits=false,
            	legend style={at={(0.6,0.93)},anchor=north},
            	legend cell align=left, 
            	xlabel={Rank of diversity metric},
            	ylabel={Selection probability},
            	scale=0.8]
				\pgfplotsset{samples at={1, ..., 10}}
				\addplot {2^(-x)};
				\addplot {x^(-2)/1.55};
				\addplot {x^(-1)/2.93};
			\legend{Exponential: $\sim 2^{-i}$, Power law: $\sim 1/i^2$, Harmonic: $\sim 1/i$}
			\end{axis}
		\end{tikzpicture}
        \caption{Rank-based selection schemes and their selection probabilities.}
		\label{fig:probtails}
\end{figure}
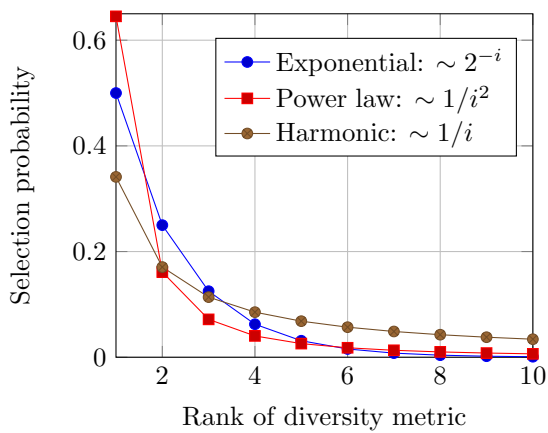
\end{definition}

Secondly, we use the classical tournament selection, but with a specific tournament size of~$\mu$, the current size of the population. This means we choose $\mu$ individuals uniformly at random with replacement from the population and then select the individual with the highest diversity contribution from this multi-set. Selection with replacement implies that there is a chance of not selecting particular individuals, while other individuals might be picked multiple times.

Now we introduce the diversity-based parent selection into SEMO (see Algorithm~\ref{alg:semopardiv}) and GSEMO (see Algorithm~\ref{alg:gsemopardiv}). Instead of using uniform parent selection, we estimate the diversity contribution for all the individuals in the population, and a parent is selected according to the diversity-based parent selection method. Then we continue as in the original algorithms. Our parent selection mechanisms are not limited to these algorithms and may prove useful on a much broader class of MOEAs.

\begin{algorithm}[!ht]
  \begin{algorithmic}[1]
  	\STATE Choose an initial solution $s\in\{0,1\}^n$ uniformly at random.
    \STATE Determine $f(s)$ and initialize $P:=\{s\}$.
    \WHILE{stopping criterion \NOT met}
    	\STATE Estimate diversity contribution $\forall s \in P$. \label{alg:semo:con}
    	\STATE Choose $s\in P$ according to the parent selection mechanism. \label{alg:semo:par}
        \STATE Choose $i \in \{1,\ldots,n\}$ uniformly at random. \label{alg:semo:chi}
        \STATE Define $s'$ by flipping the $i$-th bit of $s$. \label{alg:semo:mut}
        \IF{$s'$ is \NOT dominated by any individual in $P$} 
        \STATE Add $s'$ to $P$, and remove all individuals weakly dominated by $s'$ from~$P$. \ENDIF
    \ENDWHILE
  \end{algorithmic}
  \caption{SEMO with diversity-based parent selection}
  \label{alg:semopardiv}
\end{algorithm}

\begin{algorithm}[!ht]
  \begin{algorithmic}[1]
  	\STATE Choose an initial solution $s\in\{0,1\}^n$ uniformly at random.
    \STATE Determine $f(s)$ and initialize $P:=\{s\}$.
    \WHILE{stopping criterion \NOT met}
    	\STATE Estimate diversity contribution $\forall s \in P$.
    	\STATE Choose $s\in P$ according to the parent selection mechanism. 
        \STATE{Create $s'$ by flipping each bit in $s$ independently with probability $1/n$.} 
        \IF{$s'$ is \NOT dominated by any individual in $P$}
        	\STATE Add $s'$ to $P$, and remove all individuals weakly dominated by $s'$ from~$P$.
        \ENDIF
    \ENDWHILE
  \end{algorithmic}
  \caption{GSEMO with diversity-based parent selection}
  \label{alg:gsemopardiv}
\end{algorithm}

\section{On Diversity-Based Progress}
\label{sec:progress}

We show that diversity-based parent selection mechanisms can achieve a fast spread on the Pareto front. The following arguments and analyses consider the situation where the population is located on the Pareto front. This is trivially the case for \omm as all search points are Pareto-optimal. For \lotz we later supply a separate analysis that covers the process of reaching the Pareto front.

For \omm and \lotz the most promising parents are those that have a Hamming neighbour that is on the Pareto set, but not yet contained in the population. We call these search points \emph{\good}:
\begin{definition}[\good individuals]
\label{def:good}
With reference to a population~$P$ and a fitness function with Pareto front $\front$ and corresponding Pareto set $\paretoset$, we call a search point $x \in P \cap \paretoset$ \emph{\good} if there is a Hamming neighbour $y$ of $x$ such that $y \in \paretoset$ but $f(y) \not \in f(P)$ where $f(P)$ denotes the set of objective vectors of population $P$. Otherwise, $x$ is called \emph{\bad}.
\end{definition}

A diversity measure should encourage the selection of such \good individuals.
\begin{definition}[\property]
\label{def:prop}
We call a measure $\score{x}{P}$ \emph{\property} on $S \subseteq \{0, 1\}^n$ with respect to a fitness function with Pareto front $\front$ if for all populations $P$ and all $x, y \in P \cap \paretoset \cap S$ we have the following: if $x$ is \bad and $y$ is \good then $\score{x}{P} < \score{y}{P}$.
\end{definition}

Note that the definition is restricted to a subset $S$ of the search space. The reason is to allow the exclusion of certain search points for which the property is not true. For \omm and \lotz, the property does not hold for the extreme points on the Pareto front, $0^n$ and $1^n$. We show that both HVC and CDC are both \property on all other search points. For HVC we assume that the reference point is dominated by $(-1, -1)$. In other words, the reference point can be any point $(r_1, r_2)$ with $r_1 \le -1$ and $r_2 \le -1$.

\begin{lemma}
\label{lem:hypergood}
The hypervolume contribution $\HVC{x}{P}$ is \property on $\{0, 1\}^n \setminus \{0^n, 1^n\}$ for both \omm and \lotz if the reference point is dominated by $(-1, -1)$.
\end{lemma}
\begin{proof}
Let us consider an individual $x_i \notin \{0^n,1^n\}$ of the sorted population according to $f_1$, using the notation from Definition~\ref{def:hypcon}. If $x_i$ is bad, then there are Hamming neighbours $x_{i-1}$ and $x_{i+1}$ of $x_i$ in $P$, the $\HVC{x_i}{P}$ is the minimum possible, since $\f{1}{x_i}-\f{1}{x_{i-1}}=1$ and $\f{2}{x_{i}}-\f{2}{x_{i+1}}=1$ yielding $\HVC{x_i}{P} = (\f{1}{x_i}-\f{1}{x_{i-1}}) \cdot (\f{2}{x_{i}}-\f{2}{x_{i+1}}) = 1$.

Now, let us consider a good search point $y_i$, that is, $y_{i-1}$ or $y_{i+1}$ is not a Hamming neighbour of $y_i$. Then we have $\f{1}{y_i}-\f{1}{y_{i-1}} > 1$ or $\f{2}{y_{i}}-\f{2}{y_{i+1}} > 1$ and in any case $\HVC{y_i}{P} = (\f{1}{y_i}-\f{1}{y_{i-1}}) \cdot (\f{2}{y_{i}}-\f{2}{y_{i+1}}) > 1$. Thus $\HVC{y_i}{P} > \HVC{x_i}{P}$, which completes the proof.\qedhere
\end{proof}

\begin{lemma}
\label{lem:cdcgood}
The crowding distance contribution $\CDC{x}{P}$ is \property on $\{0, 1\}^n \setminus \{0^n, 1^n\}$ for both \omm and \lotz.
\end{lemma}
\begin{proof}
By Algorithm~\ref{alg:crowdist} the search points with the minimum and maximum $f_1$ score in the population are going to have infinite diversity score, regardless of the objective chosen to sort the population. 

Let us say that there is a bad individual $x_i$ with Hamming neighbours $x_{i-1}$ and $x_{i+1}$ contained in $P$. According to the numerator of Line~\ref{alg:crowdist:dist} of Algorithm~\ref{alg:crowdist}, the difference between the $\f{1}{x_{i-1}}$ (or $\f{2}{x_{i-1}}$) and $\f{1}{x_{i+1}}$ is the minimum possible, which means the minimum $\CDC{x_i}{P}$ is assigned to the individual~$x_i$.

In the case of a good search point $y_i$, that is, $y_{i-1}$ or $y_{i+1}$ are not Hamming neighbours of $y_i$, the difference between the next contained search points in $P$ is higher. If the difference between $\f{1}{y_i}$ (or $\f{2}{y_i}$) is higher than the minimum possible, this means $\CDC{x_i}{P} < \CDC{y_i}{P}$ which completes the proof.
\end{proof}

Note that in both above measures $0^n$ and $1^n$, if contained in the population, will always receive a high score, regardless of whether they are \good or \bad. If they are \bad, there is a high chance that a \bad individual will be selected as parent in a diversity-based parent selection mechanism. With this in mind, the probability of selecting a \good individual can be bounded from below as follows.

\begin{lemma}
\label{lem:r1-r2-r3}
Let $\score{x}{P}$ be a \property measure on $\{0, 1\}^n \setminus \{0^n, 1^n\}$. Consider either \omm or \lotz and assume the population~$P$ is a subset of the Pareto set, $P \subseteq \paretoset$. Imagine $P$ being sorted according to non-increasing $\score{x}{P}$ values. Consider a parent selection mechanism based on $\score{x}{P}$ such that $r_i$ is the probability of selecting the $i$-th element of~$P$ in the sorted sequence. Then the probability of selecting a \good individual is at least $\min\{r_1, r_2, r_3\}$ unless $P$ already covers the Pareto front.
\end{lemma}
\begin{proof}
Before the whole Pareto front is covered by the population $P$, there exists at least one \good individual $x$ in population $P$ with no corresponding Hamming neighbour $s$ in the Pareto set $\paretoset$. Then the individuals which correspond to the Hamming neighbours of the missing point $s$ are \good search points.

Since $\score{x}{P}$ is defined as a \property measure on $\{0, 1\}^n \setminus \{0^n, 1^n\}$, the \good search points have higher contribution than \bad search points that are neither $0^n$ nor $1^n$. Therefore, among the top three ranked elements in $P$, there exists at least one \good individual. The probability of selecting this \good individual is at least $\min\{r_1,r_2,r_3\}$.
\end{proof}

The parent selection mechanisms thus have the following probability of selecting \good individuals.
\begin{lemma}
\label{lem:parent-selection-property-for-good-individual}
In the setting described in Lemma~\ref{lem:r1-r2-r3}, the probability $\pgood$ of selecting a good individual is
\begin{enumerate}
\item $\Omega(1)$ for the exponential and power law ranking schemes,
\item $\Omega(1/\log \mu)$ for the harmonic ranking scheme,
\item $\Omega(1)$ for tournament selection with tournament size~$\mu$.
\end{enumerate}
\end{lemma}
\begin{proof}
For the parent selection with the exponential ranking scheme, the probability follows from Lemma~\ref{lem:r1-r2-r3}, which fulfils
\[
r_1 \ge r_2 \ge r_3 = \frac{2^{-3}}{\displaystyle\sum_{j=1}^\mu 2^{-j}} \ge 2^{-3} = \Omega(1).
\]

For the power law ranking scheme, since $\sum_{j=1}^{\mu}\frac{1}{j^2} \le \sum_{j=1}^{\infty}\frac{1}{j^2} = \pi^2/6$, the probability fulfils
\[
r_1 \ge r_2 \ge r_3 = \frac{1/3^2}{\displaystyle\sum_{j=1}^{\mu} \frac{1}{j^2}} \ge \frac{2}{3\cdot\pi^2} = \Omega(1).
\]

In the case of the harmonic ranking scheme, since $\sum_{j=1}^{\mu}\frac{1}{j} \le \ln \mu + 1$, the probability fulfils
\[
r_1 \ge r_2 \ge r_3 = \frac{1/3}{\displaystyle\sum_{j=1}^{\mu} \frac{1}{j}} \ge \frac{1}{3 \cdot(\ln \mu + 1)} = \Omega(1/\log \mu).
\]

For tournament selection, the probability of selecting a good individual is at least $\min\{r_1,r_2,r_3\}$ and $r_1 \ge r_2 \ge r_3$. In order for the individual with the 3rd maximum contribution to be selected in the tournament selection, the individuals with the 1st and 2nd maximum contribution should never be selected in the $\mu$ times (probability of $\left(1-2/\mu\right)^\mu$). And, conditional on this happening, the individual with the 3rd maximum contribution has to be chosen at least once amongst the other $\mu-2$ individuals in the $\mu$ times with probability $1-\left(1-\frac{1}{\mu-2}\right)^\mu$. Hence, the probability of selecting a good individual is at least
\[
\pgood \ge \left(1-\left(1-\frac{1}{\mu-2}\right)^{\mu}\right)\cdot \left(1-\frac{2}{\mu}\right)^{\mu}
\ge \left(1-\frac{1}{e}\right) \cdot \left(1-\frac{2}{\mu}\right)^{\mu}
\]
using $\left(1-\frac{1}{x}\right)^x \le 1/e$ for $x > 1$.
Since $f(x)=\left(1-\frac{1}{x}\right)^x$ is non-decreasing when $x\geq 1$, with $\mu\geq 3$, $\left(1-\frac{2}{\mu}\right)^{\frac{\mu}{2}} \ge \left(1-\frac{2}{3}\right)^{\frac{3}{2}}\ge 0.19.$ Therefore, $\pgood \ge \left(1-\frac{1}{e}\right)\cdot 0.19^2 =\Omega(1).$
\end{proof}

\section{Speedups on \omm}
\label{sec:ommpro}

For any parent selection mechanism defined before, the parent selection is focused on selecting an individual with a high diversity score. In the case of HVC or CDC, having a high diversity contribution means that, apart from the possible exceptions of $0^n$ and $1^n$, the parent will be \good, \ie located in a less populated area of the Pareto front. We show that by preferring \good individuals in the parent selection, SEMO and GSEMO can quickly find the whole Pareto front for \omm.

\begin{lemma}
\label{lem:oneminmax}
Suppose that the probability of selecting a \good individual is at least $\pgood$. Then the expected runtime for SEMO or GSEMO to find all solutions in the Pareto front on \omm is bounded above by $O((n \log{n}) / \pgood)$.
\end{lemma}
\begin{proof}
We call a step a \emph{relevant step} if the algorithm selects a \good parent on the Pareto front. We show in the following that $O(n \log n)$ relevant steps are sufficient for covering the whole Pareto front of \omm, regardless of irrelevant steps performed. This shows the claim as the expected time for a relevant step is $1/\pgood$.

We use the \emph{accounting method} (see, \eg Section~17.2 in~\cite{Cormen2009}) to bound the number of relevant steps. Specifically, we count the number of relevant steps spent in selecting a \good parent with $i$ ones. Summing up (upper bounds on) all these times across all~$0 \le i \le n$ will imply the claim.

Note that, once potential gaps at $i-1$ and $i+1$ are filled, there can be no more relevant steps at $i$ ones, due to the definition of a relevant step. Hence the expected number of relevant steps at $i$ ones is bounded by the expected number of mutations from~$i$ needed to fill both these gaps. If an individual with $i$ ones, $0 < i < n$, is selected as parent, the probability of mutation creating an individual with $i-1$ ones is at least $i/n \cdot (1-1/n)^{n-1} \ge i/(en)$, and the probability of mutation creating an individual with $i+1$ ones is at least $(n-i)/n \cdot (1-1/n)^{n-1} \ge (n-i)/(en)$ (this holds both for SEMO and GSEMO; for SEMO the factor $1/e$ can be removed). The time for filling both gaps is at most $en/i + en/(n-i)$. Hence there are at most $en/i + en/(n-i)$ relevant steps selecting a parent with $i$ ones. In the special cases of $i=0$ or $i=n$ the time to fill the neighbouring gaps simplifies to $en/n=e$.

Summing over all~$i$, the expected total number of relevant steps is hence at most
\[
2e + \sum_{i=1}^{n-1} \left(\frac{en}{i} + \frac{en}{n-i}\right)
= 2e + 2\sum_{i=1}^{n-1} \frac{en}{i}
= 2\sum_{i=1}^{n} \frac{en}{i}
\le 2en (\log{n}+1).
\]
Where the summation $\harm=\sum_{i=1}^{n}1/i$ is known as the \emph{harmonic number} and satisfies $\harm=\ln{n} + \Theta(1)$ this completes the proof.
\end{proof}

Combining Lemma~\ref{lem:parent-selection-property-for-good-individual} and Lemma~\ref{lem:oneminmax}, we have proved the following results. Note that the population size $\mu$ is always at most $n+1$ on \omm and \lotz, hence for the harmonic ranking scheme, $\pgood = \Omega(1/\log \mu) = \Omega(1/\log n)$.
\begin{theorem}
Consider SEMO and GSEMO with diversity-based parent selection using any diversity measure that is diversity-favouring on $\{0, 1\}^n \setminus \{0^n, 1^n\}$ (e.\,g.\ HVC or CDC). Then the expected time to find the whole Pareto front on \omm is bounded by $O(n \log n)$ for the exponential and power law ranking schemes, and for tournament selection with tournament size~$\mu$. It is bounded by $O(n \log^2 n)$ for the harmonic ranking scheme.
\end{theorem}
As both SEMO and GSEMO with the classical uniform parent selection need time $\Theta(n^2 \log n)$ on \omm, our parent selection schemes lead to speedups of order $\Theta(n)$ and $\Theta(n/\log n)$, respectively.

\section{Speedups on \lotz}
\label{sec:lotzpro}

We now turn to the function \lotz. In contrast to \omm, where all individuals are Pareto optimal, for \lotz we have to estimate the time for the population to reach the Pareto front. For SEMO the approach to the Pareto front can be estimated easily since SEMO keeps only one individual in the population. For local mutations as used in SEMO, whenever an offspring is created, either the offspring dominates the parent, or the parent dominates the offspring (or both, if they have the same function values). The population size remains unchanged before there is a solution on the Pareto front. For any parent on the Pareto front, SEMO only accepts its offspring if it is also on the Pareto front, otherwise the offspring is dominated by the parent.

\begin{lemma}
\label{lem:semlotz}
The expected time for SEMO to reach the Pareto front is $O(n^2)$. Assume that afterwards the probability of selecting a \good individual in the population is at least $\pgood$. The expected runtime for SEMO to reach a population covering the whole Pareto front on \lotz is bounded above by $O(n^2/\pgood)$.
\end{lemma}
\begin{proof}
The time for the population to find the first Pareto-optimal point is $O(n^2)$ and has already been proved in Lemma~1 in \cite{Laumanns2004}. So we can focus on the time required to find the whole Pareto front. By the \emph{accounting method} used to prove Lemma~\ref{lem:oneminmax} and the definition of relevant step: the algorithm selects a \good parent on the Pareto front,  we count the number of relevant steps spent selecting a \good parent with $i$ leading ones, $1^i0^{n-i}$, and sum up all these times across all $0\leq i \leq n$ to prove the claim.

The potential gaps consist of non-existing non-dominated individuals at $i-1$ and $i+1$ ($1^{i-1}0^{n-i+1}$ and $1^{i+1}0^{n-i-1}$, respectively). It is necessary to fill those gaps by including these search points in the population. Once this has happened, there can be no more relevant steps at $i$ leading ones. So the expected number of mutations at $i$ leading ones is bounded by the expected number of mutations from $i$ needed to fill $i-1$ and $i+1$. If $1^i0^{n-i}$ is selected as parent, the probability of mutation creating $1^{i-1}0^{n-i+1}$ or $1^{i+1}0^{n-i-1}$ is $1/n$, respectively. The time for filling both gaps (if existent) is at most $n + n$. Hence there are in expectation at most $2n$ relevant steps selecting a parent with $i$ leading ones.

Summing over all $i$, the expected total number of relevant steps is hence at most
\[
\sum_{i=0}^{n} 2n= 2n(n+1)=O(n^2).
\]

Noting that the expected waiting time for a relevant step is $1/\pgood$. Thus the overall expected runtime for SEMO to achieve a population covering the whole Pareto front on \lotz is upper bounded by $O(n^2) + O(n^2/\pgood) = O(n^2/\pgood)$.
\end{proof}

Combining Lemma~\ref{lem:parent-selection-property-for-good-individual} and Lemma~\ref{lem:semlotz},
we now have proved the following results.

\begin{theorem}
Consider SEMO with diversity-based parent selection using any diversity measure that is diversity-favouring on $\{0, 1\}^n \setminus \{0^n, 1^n\}$ (e.\,g.\ HVC or CDC). Then the expected time to find the whole Pareto front on \lotz is bounded by $O(n^2)$ for the exponential and power law ranking schemes, and for tournament selection with tournament size~$\mu$. It is bounded by $O(n^2 \log n)$ for the harmonic ranking scheme.
\end{theorem}

The analysis of GSEMO turns out to be more difficult than the analysis of SEMO. The reason is that the approach to the Pareto front becomes harder to analyse. With global mutations, GSEMO can create incomparable search points while approaching the Pareto front. This means that the population can expand in size while approaching the Pareto front, and even after the whole population has reached the Pareto front, it is possible to create search points off the Pareto front that are accepted in the population.

Experiments in Section~\ref{sec:exp} indicate that this behaviour does not slow down the algorithm by more than a constant factor. However, proving that the bound $O(n^2)$ for SEMO also holds for GSEMO turns out to be very challenging. We therefore take a different approach and analyse a modified variant of GSEMO that is easier to analyse. Experiments presented in Section~\ref{sec:exp} confirm that this modification does not significantly change the average runtime (inspecting Tables~\ref{tab:explotz} and \ref{tab:explranlotz}, the quotients of average times for the modified GSEMO and those for the original GSEMO across all parent selection mechanisms are $0.88$ for $\HVC{-1}{-1}$, $1.19$ for $\HVC{-n}{-n}$, and $1.27$ for CDC, averaging to $1.1$ are close to~1 in many settings and always in the interval $[0.48, 2.03]$).

The idea behind this modification is to simplify the approach to the Pareto front by restricting parent selection to search points that are maximal with regards to a linear combination of both objectives. 
\begin{definition}[$\ld$-dominant attribute]
\label{def:ldom}
Let $\LD{x}=\LO{x}+\TZ{x}$, where $\LO{x}$ and $\TZ{x}$ denotes the total number of leading ones and the total number of trailing zeros of a certain individual $x$, respectively.
\end{definition}

We modify GSEMO in such a way that it only picks parents with maximal $\ld$\nobreakdash-dominant attribute in the population (see Algorithm~\ref{alg:mgsemo}), and also the computation of the diversity contribution is restricted to these search points. This has two effects: it simplifies and facilitates the analysis of the individuals while they are approaching the Pareto front. While the original GSEMO can store incomparable search points with different $\ld$\nobreakdash-values in the population, the modified GSEMO only considers incomparable search points with maximum $\ld$\nobreakdash-value. In addition, since all $x$ individuals on the Pareto front have the largest possible value of $\LD{x} = n$, once the Pareto front is reached, the algorithm only selects individuals on the Pareto front as parents according to their diversity contribution. 

We first bound the expected time to reach the Pareto front.

\begin{algorithm}[tb]
  \begin{algorithmic}[1]
  	\STATE Choose an initial solution $s\in\{0,1\}^n$ uniformly at random.
    \STATE Determine $f(s)$ and initialize $P:=\{s\}$.
    \WHILE{stopping criterion \NOT met}
        \STATE Let $P' \subseteq P$ be the set of all search points with a maximum $\ld$-dominant attribute in~$P$.\STATE Estimate diversity contribution $\forall s \in P'$ w.\,r.\,t.\ the population~$P'$.
        \STATE Choose $s\in P'$ according to parent selection mechanism.
        \STATE Create $s'$ by flipping each bit of $s$ independently with probability $1/n$.
        \IF{$s'$ is \NOT dominated by any individual in $P$} 
        	\STATE Add $s'$ to $P$, and remove all individuals weakly dominated by $s'$ from~$P$. 
        \ENDIF
    \ENDWHILE
  \end{algorithmic}
  \caption{Modified Global SEMO with diversity-based parent selection}
  \label{alg:mgsemo}
\end{algorithm}

\begin{lemma}
\label{lem:lotz_g1}
The expected time for the modified GSEMO to reach the Pareto front is bounded above by $O(n^2)$.
\end{lemma}
\begin{proof}
According to Definition~\ref{def:ldom}, before reaching the Pareto front, the solution with $\max_{x\in P}(\LD{x})$ is selected to generate an offspring. Consider the event of only flipping the first $0$-bit or the last $1$-bit of the selected individual. Since the offspring from this event has a higher value of one of the objectives than its parent which is of the maximum $\LD{x}$ in the population, the offspring is non-dominated by any individuals in the population and is accepted by the algorithm. Hence, the probability of increasing $\max_{x\in P}(\LD{x})$ is at least
\[
2\cdot\frac{1}{n}\cdot\left(1-\frac{1}{n}\right)^{n-1}\geq\frac{2}{en}.
\]

Throughout the process, the value of $\max_{x\in P}(\LD{x})$ in the population never goes down. Therefore, the overall expected runtime for GSEMO with this selection scheme to reach the Pareto front is at most
\[
\sum_{\ld_{\max}=0}^{n-2}\frac{en}{2}=O(n^2).\qedhere
\]
\end{proof}

\begin{lemma}
\label{lem:mgsemolotz}
Assume that the probability of selecting a good individual in the population is at least $\pgood$. The expected time for the modified GSEMO to reach a population covering the whole Pareto front on \lotz is bounded above by $O(n^2/\pgood)$.
\end{lemma}
\begin{proof}
As for SEMO, before the population covers the whole Pareto front, the optimisation process of the modified GSEMO can be divided into two stages. The first stage focusses on obtaining the first individual on the Pareto front and the second one focusses on covering the Pareto front. As proved in Lemma~\ref{lem:lotz_g1}, the expected time for the modified GSEMO to reach the Pareto front is at most $O(n^2)$.

In the second stage, by following the definition of relevant step, the parent to be selected is a \good search point on the Pareto front with the maximum $\LD{x}$ dominant attribute. The algorithm will select individuals on the Pareto front with the maximum $\LD{x}$ dominant attribute according to their diversity contribution. So now we can apply the accounting method used to prove previous lemmas to bound the number of relevant steps spent selecting the \good parent.

As in Lemma~\ref{lem:semlotz}, we define a \good parent with $i$ leading ones with possible gaps on $i-1$ and/or $i+1$ across all $0\leq i\leq n$. And by introducing the factor $1/e$ to the analysis in Lemma~\ref{lem:semlotz}, we now have the time for filling both gaps is at most $1/(en) + 1/(en)$. Hence there are at most $en + en = 2en$ relevant steps selecting a \good parent with $i$ leading ones. Summing over all $i$, the expected total number of relevant steps is hence at most
\[
\sum_{i=0}^{n}2en=2e\sum_{i=0}^{n}n=O(n^2)
\]
The overall runtime for the modified GSEMO on \lotz to reach a population covering the whole Pareto front is bounded above by $O(n^2/\pgood)$.
\end{proof}

As mentioned on the proof of the previous lemma, once the individual with the maximum $\LD{x}$ dominant attribute has reached the Pareto front, the algorithm will always select \good individuals on the Pareto front (with the maximum $\LD{x}$ dominant attribute) according to their diversity contribution. This characteristic allows us to apply Lemma~\ref{lem:parent-selection-property-for-good-individual}, and by Lemma~\ref{lem:mgsemolotz}, we now have proved the following results.

\begin{theorem}
Consider the modified GSEMO with diversity-based parent selection using any diversity measure that is diversity-favouring on $\{0, 1\}^n \setminus \{0^n, 1^n\}$ (e.\,g.\ HVC or CDC). Then
the expected time to find the whole Pareto front on \lotz is bounded by $O(n^2)$ for the exponential and power law ranking schemes, and for tournament selection with tournament size~$\mu$. It is bounded by $O(n^2 \log n)$ for the harmonic ranking scheme.
\end{theorem}

\section{Experiments}
\label{sec:exp}

The experimental approach is focused on the analysis of SEMO, GSEMO and the modified GSEMO and their performance with and without the diversity-based parent selection mechanisms. We are interested in observing if we can speed up the performance from the classical approaches. For the case of the modified GSEMO, we measure its performance only on \lotz and we compare its performance to GSEMO in order to observe the impact of the $\ld$-dominant attribute on the performance of the algorithm.

Experiments also allow for a more detailed comparison of the HVC, CDC, and the parent selection methods. In the case of the HVC, we have defined two settings for the reference points, $(-1,-1)$ and $(-n,-n)$. For the first reference point, a slight preference to the extreme points is provided while with the second, the influence of the extreme points becomes very strong. This particular characteristic became an interesting feature to observe in the case of the ranking-based selection schemes, and exposes a potential flaw for the case of HVC with low (or high in the case of minimisation) reference point or CDC (since it assigns infinite value to the extreme points) and the parent selection mechanisms that focus very aggressively toward the extreme points, as we shall see below.

Since we are interested in the time required to find the Pareto front, we report the following outcomes and stopping criteria for each run. \emph{Success}, the whole Pareto front has been covered, \ie the run is stopped if the population contains all individuals on the Pareto front. \emph{Failure/Stagnation}, once the run has reached 1 million generations and the Pareto front has not been fully covered, this is enough time for the algorithms to create new individuals and fill the gaps on the Pareto front. We repeat the experimental framework for $100$ runs with problem size $n=100$ for all algorithmic approaches and report the mean and standard deviation (STD) as our metrics of interest.

Table~\ref{tab:expsalg} shows the mean and STD of generations required to find the Pareto front for the classic SEMO and GSEMO that use uniform parent selection for both test functions. Table~\ref{tab:expomm} and \ref{tab:explotz} refer to the mean and standard deviation of generations required to find the Pareto front for SEMO and GSEMO with the different diversity-based parent selection schemes for \omm and \lotz, respectively. Finally, Table~\ref{tab:explranlotz} shows the mean and standard deviation of generations required to find the Pareto front for the modified GSEMO on \lotz.

\begin{table*}
\caption{Mean (first rows) and STD (second rows) of generations required to find the Pareto front for SEMO and GSEMO on \omm and \lotz with $n=100$.}
  \label{tab:expsalg}
  \centering
\begin{tabular}{lcc}
\toprule
    Algorithms & \omm & \lotz \\
   \midrule
    \mrow{2}{*}{SEMO} & $\expnumber{4.16}{04}$ & $\expnumber{3.17}{05}$ \\
    & $\expnumber{1.15}{04}$ & $\expnumber{5.34}{04}$ \\
    \midrule
    \mrow{2}{*}{GSEMO} & $\expnumber{1.06}{05}$ & $\expnumber{6.58}{05}$ \\
    & $\expnumber{3.47}{04}$ & $\expnumber{1.12}{05}$ \\
    \bottomrule
\end{tabular}
\end{table*}

As we mentioned before, a parent selection mechanisms that is extremely focused on the extreme points can be potentially dangerous, and to exemplify this, we have introduced a deterministic selection mechanism which we have named \emph{Highest Diversity Contribution} (HDC): always select an individual with the highest diversity contribution (break ties uniformly at random if there are several such points). We also have defined a modified version of the uniform random selection used by SEMO and GSEMO, that we call \emph{Non-Minimum Uniform at Random} (NMUAR), where the individuals with the minimum diversity contribution in the population are ignored (provided that the population does contain multiple diversity contribution values) and one individual is selected uniformly at random from all remaining individuals. In this sense individuals with high diversity contributions have better probabilities to be selected and the approach is flexible enough to choose between extreme and intermediate individuals.

As it can be observed in Table~\ref{tab:expomm} and \ref{tab:explotz}, HDC fails to find the Pareto front for \omm and \lotz in the case of GSEMO for both diversity-based metrics. For the case of GSEMO with HDC selection mechanism with HVC and CDC on \omm the failure rate was $0.94$ and $0.93$, respectively. On \lotz, the failure rate was $1.0$ for both diversity metrics.

\begin{table*}
  \caption{Mean (first rows) and STD (second rows) of generations required to find the Pareto front for SEMO and GSEMO with diversity-based parent selection methods on \omm with $n=100$. ``Stagnation'' indicates a failure rate larger than 0.}
  \label{tab:expomm}
  \centering
  \begin{tabular}{lccc}
    \toprule
    {\bf Algorithms} & \boldmath$\HVC{-1}{-1}$ & \boldmath$\HVC{-n}{-n}$ & \boldmath$\cdc$ \\
    \midrule
    \mrow{2}{*}{SEMO \& HDC} & $\expnumber{9.14}{02}$ & $\expnumber{8.90}{02}$ & $\expnumber{1.05}{03}$ \\
    & $\expnumber{1.76}{02}$ & $\expnumber{1.65}{02}$ & $\expnumber{2.40}{02}$ \\
    \midrule
     \mrow{2}{*}{GSEMO \& HDC} & $\expnumber{2.12}{03}$ & Stagnation & Stagnation \\
     & $\expnumber{4.28}{02}$ & Stagnation & Stagnation \\
     \midrule
     \mrow{2}{*}{SEMO \& NMUAR} & $\expnumber{8.92}{02}$ & $\expnumber{1.05}{03}$ & $\expnumber{1.03}{03}$ \\
     & $\expnumber{1.81}{02}$ & $\expnumber{2.72}{02}$ & $\expnumber{2.59}{02}$ \\
     \midrule
     \mrow{2}{*}{GSEMO \& NMUAR} & $\expnumber{2.14}{03}$ & $\expnumber{2.54}{03}$ & $\expnumber{2.58}{03}$ \\
     & $\expnumber{4.97}{02}$ & $\expnumber{6.57}{02}$ & $\expnumber{7.86}{02}$ \\
     \midrule
 	\mrow{2}{*}{SEMO \& exponential} & $\expnumber{1.28}{03}$ & $\expnumber{1.27}{03}$ & $\expnumber{1.36}{03}$ \\
     & $\expnumber{2.72}{02}$ & $\expnumber{2.71}{02}$ & $\expnumber{3.44}{02}$ \\
     \midrule
     \mrow{2}{*}{GSEMO \& exponential} & $\expnumber{3.21}{03}$ & $\expnumber{3.18}{03}$ & $\expnumber{3.24}{03}$ \\
     & $\expnumber{9.35}{02}$ & $\expnumber{9.12}{02}$ & $\expnumber{7.72}{02}$ \\
     \midrule
 	\mrow{2}{*}{SEMO \& harmonic} & $\expnumber{3.05}{03}$ & $\expnumber{3.24}{03}$ & $\expnumber{3.28}{03}$ \\
     & $\expnumber{6.97}{02}$ & $\expnumber{8.63}{02}$ & $\expnumber{8.03}{02}$ \\
     \midrule
     \mrow{2}{*}{GSEMO \& harmonic} & $\expnumber{7.89}{03}$ & $\expnumber{7.26}{03}$ & $\expnumber{8.03}{03}$ \\
     & $\expnumber{1.90}{03}$ & $\expnumber{1.69}{03}$ & $\expnumber{2.09}{03}$ \\
     \midrule
 	\mrow{2}{*}{SEMO \& power law} & $\expnumber{1.15}{03}$ & $\expnumber{1.24}{03}$ & $\expnumber{1.34}{03}$ \\
     & $\expnumber{2.48}{02}$ & $\expnumber{2.89}{02}$ & $\expnumber{3.00}{02}$ \\
     \midrule
     \mrow{2}{*}{GSEMO \& power law} & $\expnumber{2.87}{03}$ & $\expnumber{2.85}{03}$ & $\expnumber{3.32}{03}$ \\
     & $\expnumber{6.35}{02}$ & $\expnumber{6.22}{02}$ & $\expnumber{1.07}{03}$ \\
     \midrule
 	\mrow{2}{*}{SEMO \& tournament($\mu$)}\!\!\!\!\!\!\!\!\!\! & $\expnumber{1.05}{03}$ & $\expnumber{1.08}{03}$ & $\expnumber{1.21}{03}$ \\
     & $\expnumber{2.24}{02}$ & $\expnumber{2.18}{02}$ & $\expnumber{3.09}{02}$ \\
     \midrule
     \mrow{2}{*}{GSEMO \& tournament($\mu$)}\!\!\!\!\!\!\!\!\!\! & $\expnumber{2.58}{03}$ & $\expnumber{2.60}{03}$ & $\expnumber{2.81}{03}$ \\
     & $\expnumber{5.48}{02}$ & $\expnumber{7.91}{02}$ & $\expnumber{7.34}{02}$ \\
    \bottomrule
  \end{tabular}
\end{table*}

\begin{table*}
  \caption{Mean (first rows) and STD (second rows) of generations required to find the Pareto front for SEMO and GSEMO with diversity-based parent selection methods on \lotz with $n=100$. ``Stagnation'' indicates a failure rate larger than 0.}
  \label{tab:explotz}
  \centering
  \begin{tabular}{lccc}
    \toprule
    {\bf Algorithms} & \boldmath$\HVC{-1}{-1}$ & \boldmath$\HVC{-n}{-n}$ & \boldmath$\cdc$ \\
     \midrule
     \mrow{2}{*}{SEMO \& HDC} & $\expnumber{1.24}{04}$ & $\expnumber{1.25}{04}$ & $\expnumber{1.41}{04}$ \\
     & $\expnumber{9.79}{02}$ & $\expnumber{1.22}{03}$ & $\expnumber{1.80}{03}$ \\
     \midrule
     \mrow{2}{*}{GSEMO \& HDC} & $\expnumber{3.06}{04}$ & Stagnation & Stagnation \\
     & $\expnumber{2.62}{03}$ & Stagnation & Stagnation \\
     \midrule
     \mrow{2}{*}{SEMO \& NMUAR} & $\expnumber{1.25}{04}$ & $\expnumber{1.38}{04}$ & $\expnumber{1.41}{04}$ \\
     & $\expnumber{1.10}{03}$ & $\expnumber{1.49}{03}$ & $\expnumber{1.52}{03}$ \\
     \midrule
     \mrow{2}{*}{GSEMO \& NMUAR} & $\expnumber{3.17}{04}$ & $\expnumber{3.50}{04}$ & $\expnumber{3.58}{04}$ \\
     & $\expnumber{3.13}{03}$ & $\expnumber{3.85}{03}$ & $\expnumber{3.75}{03}$ \\
     \midrule
 	\mrow{2}{*}{SEMO \& exponential} & $\expnumber{1.57}{04}$ & $\expnumber{1.58}{04}$ & $\expnumber{1.78}{04}$ \\
     & $\expnumber{1.31}{03}$ & $\expnumber{1.33}{03}$ & $\expnumber{2.47}{03}$ \\
     \midrule
     \mrow{2}{*}{GSEMO \& exponential} & $\expnumber{3.45}{04}$ & $\expnumber{4.00}{04}$ & $\expnumber{5.87}{04}$ \\
     & $\expnumber{2.87}{03}$ & $\expnumber{8.60}{03}$ & $\expnumber{1.63}{04}$ \\
     \midrule
 	\mrow{2}{*}{SEMO \& harmonic} & $\expnumber{3.14}{04}$ & $\expnumber{3.08}{04}$ & $\expnumber{3.53}{04}$ \\
     & $\expnumber{3.60}{03}$ & $\expnumber{3.24}{03}$ & $\expnumber{5.68}{03}$ \\
     \midrule
     \mrow{2}{*}{GSEMO \& harmonic} & $\expnumber{6.69}{04}$ & $\expnumber{6.33}{04}$ & $\expnumber{6.73}{04}$ \\
     & $\expnumber{7.23}{03}$ & $\expnumber{7.40}{03}$ & $\expnumber{1.02}{04}$ \\
     \midrule
 	\mrow{2}{*}{SEMO \& power law} & $\expnumber{1.54}{04}$ & $\expnumber{1.51}{04}$ & $\expnumber{1.69}{04}$ \\
     & $\expnumber{1.26}{03}$ & $\expnumber{1.36}{03}$ & $\expnumber{2.13}{03}$ \\
     \midrule
     \mrow{2}{*}{GSEMO \& power law} & $\expnumber{3.40}{04}$ & $\expnumber{5.03}{04}$ & $\expnumber{5.73}{04}$ \\
     & $\expnumber{3.30}{03}$ & $\expnumber{1.24}{04}$ & $\expnumber{1.43}{04}$ \\
     \midrule
 	\mrow{2}{*}{SEMO \& tournament($\mu$)\!\!\!\!\!\!\!\!\!\!} & $\expnumber{1.38}{04}$ & $\expnumber{1.41}{04}$ & $\expnumber{1.55}{04}$ \\
     & $\expnumber{1.25}{03}$ & $\expnumber{1.12}{03}$ & $\expnumber{1.94}{03}$ \\
     \midrule
     \mrow{2}{*}{GSEMO \& tournament($\mu$)\!\!\!\!\!\!\!\!\!\!} & $\expnumber{3.16}{04}$ & $\expnumber{6.53}{04}$ & $\expnumber{7.87}{04}$ \\
     & $\expnumber{2.88}{03}$ & $\expnumber{2.15}{04}$ & $\expnumber{2.57}{04}$ \\
    \bottomrule
  \end{tabular}
\end{table*}

The reason for these bad results for GSEMO (and the modified GSEMO) on \omm is due to the mutation operator. Both algorithms can create gaps by creating an offspring that may differ from its parent with more than one bit. In the case of GSEMO on \lotz, the algorithm can create incomparable search points and the population expands in size while approaching the Pareto front. This implies that the Pareto front is reached in different areas at different times during the run, leaving intermediate unexplored regions. Once the Pareto front has been reached, the algorithm can create gaps by creating an offspring by flipping more than one leading one or trailing zero. Then it will continue selecting those individuals ignoring the intermediate ones, leaving the population in a \emph{stagnation} state. This observation also justifies why we introduced parent selection schemes of varying degree of aggressiveness. We analyse this process rigorously in Section~\ref{sec:greed}.

For all other parent selection schemes defined in this paper, we have achieved a significant speed up in the performance of SEMO and GSEMO of around one order of magnitude. As it can be observed in Table~\ref{tab:expomm} and \ref{tab:explotz}, SEMO and GSEMO with diversity-based parent selection mechanisms are able to find the Pareto front faster than its classical counterparts, \ie fewer generations are required for both test functions. Note that the problem size $n=100$ is relatively moderate; as our theoretical results prove, speedups over the original algorithms will grow further when the problem size is increased.

In the case of the modified GSEMO, the same stagnation state was reached (see Table~\ref{tab:explranlotz}). For the modified GSEMO with HDC selection mechanism with HVC and CDC on \omm the failure rate was $0.97$ and $1.0$, respectively. On \lotz the failure rate for the modified GSEMO with HVC and CDC decreases considerably, reaching $0.37$ and $0.33$, respectively. 

\begin{table*}
  \caption{Mean (first rows) and STD (second rows) of generations required to find the Pareto front for the modified GSEMO and diversity-based parent selection methods on \lotz with $n=100$. ``Stagnation'' indicates a failure rate larger than 0.}
  \label{tab:explranlotz}
  \centering
  \begin{tabular}{lccc}
   \toprule
    {\bf Algorithms} & \boldmath$\HVC{-1}{-1}$ & \boldmath$\HVC{-n}{-n}$ & \boldmath$\cdc$ \\
     \midrule
     \mrow{2}{*}{HDC} & $\expnumber{3.06}{04}$ & Stagnation & Stagnation \\
     & $\expnumber{2.78}{03}$ & Stagnation & Stagnation \\
     \midrule
     \mrow{2}{*}{NMUAR} & $\expnumber{3.19}{04}$ & $\expnumber{3.60}{04}$ & $\expnumber{3.55}{04}$ \\
     & $\expnumber{2.92}{03}$ & $\expnumber{4.50}{03}$ & $\expnumber{4.92}{03}$ \\
     \midrule
 	 \mrow{2}{*}{Exponential} & $\expnumber{3.95}{04}$ & $\expnumber{3.99}{04}$ & $\expnumber{4.55}{04}$ \\
     & $\expnumber{3.65}{03}$ & $\expnumber{3.62}{03}$ & $\expnumber{6.00}{03}$ \\
     \midrule
     \mrow{2}{*}{Harmonic} & $\expnumber{8.13}{04}$ & $\expnumber{8.11}{04}$ & $\expnumber{9.49}{04}$ \\
     & $\expnumber{8.37}{03}$ & $\expnumber{8.53}{03}$ & $\expnumber{1.50}{04}$ \\
     \midrule
     \mrow{2}{*}{Power law} & $\expnumber{3.81}{04}$ & $\expnumber{3.81}{04}$ & $\expnumber{4.32}{04}$ \\
     & $\expnumber{3.62}{03}$ & $\expnumber{3.66}{03}$ & $\expnumber{5.75}{03}$ \\
     \midrule
     \mrow{2}{*}{Tournament($\mu$)\!\!\!\!\!\!\!\!\!\!} & $\expnumber{3.46}{04}$ & $\expnumber{3.49}{04}$ & $\expnumber{3.88}{04}$ \\
     & $\expnumber{2.95}{03}$ & $\expnumber{3.42}{03}$ & $\expnumber{5.50}{03}$ \\
    \bottomrule
  \end{tabular}
\end{table*}

The modified GSEMO on \lotz achieved a considerably lower failure rate compared to the original GSEMO, where it was 1.0. 
We believe that there are two reasons for this. Firstly, for the modified GSEMO it is not possible to reach the Pareto front in different areas, avoiding the creation of gaps while approaching the Pareto front; the individual with the largest $\ld$-dominant attribute will always reach the Pareto front. Secondly, after the Pareto front has been reached, the algorithm will select individuals on the Pareto front as parents according to their diversity contribution. Here, from the $i$ individual, the mutation operator needs to flip $1^{i-1}0^{n-i+1}$ or $1^{i+1}0^{n-i-1}$ to create a new individual. In this sense it is more difficult to leave an empty space between points leading to this better performance but it is always possible for the algorithm to flip multiple consecutive bits to create a gap, resulting in the mentioned failure rates.

The modified GSEMO can also achieve a significant speed up in performance on \lotz. With this preliminary analysis we can see that the introduction of the $\ld$-dominant attribute does not drastically change the average runtime and it can be used as an approximation or first step towards the definition of a bound for GSEMO with diversity-based parent selection on \lotz.

\section{Comparing Selection Schemes: How Much Greed is Good?}
\label{sec:discgreed}

In this section we focus our attention on the Highest Diversity Contribution (HDC) and the Non-Minimum Uniformly at Random (NMUAR) methods. In Section~\ref{sec:greed} we discuss in detail how HDC seems to be the fastest selection mechanism for SEMO, but the worst for GSEMO as it leads to stagnation. We show by means of rigorous runtime analysis how this is a rare and natural example where multi-bit flips do a lot of harm by leading the population into a stagnation state.

Finally in Section~\ref{sec:disnmuars} we discuss the results obtained regarding the NMUAR mechanism. As shown in Section~\ref{sec:exp}, NMUAR performs experimentally well for SEMO and GSEMO with no stagnation outcome. We show that for a particular choice of the reference point  NMUAR can lead the population into a stagnation state. 
On the positive side, we show that NMUAR is able to efficiently optimise both \lotz and \omm for common choices of the reference point.

\subsection{Why Highest Diversity Contribution Stagnates}
\label{sec:greed}

In this section we theoretically examine the stagnation results of Section~\ref{sec:exp} related to GSEMO with the HDC selection strongly favouring the extreme points. As it can be observed from Tables~\ref{tab:expomm} and \ref{tab:explotz}, a greedy approach seems to be the best for SEMO. SEMO can find all individuals on the Pareto front but also is the fastest in doing so. This is because for SEMO on \omm, all individuals are part of the Pareto front and the algorithm starts with one individual on the Pareto front. In the case of \lotz the algorithm always reaches the Pareto front with just one individual. Once on the Pareto front, the spread of the population to outer areas can only be achieved by individuals that differ from its parent in just one bit, \ie no gaps or empty spaces are left between points. 

In the following we show by means of rigorous runtime analysis why the previous experimental results occur for the modified GSEMO on \lotz. Let the reference point be dominated by $(-n^2, -n^2)$ for the HVC in order to simplify the analysis for proving that focusing on extreme points can lead to undesired results. Our main result of this section is the following.

\begin{theorem}
\label{the:probability-of-gaps}
Consider the modified GSEMO with Highest Diversity Contribution, choosing as diversity metric either CDC or HVC with a reference point dominated by $(-n^2, -n^2)$ on the function \lotz. Then at the first point in time the population $P_t$ contains both $0^n$ and $1^n$, $P_t$ equals the whole front with probability $\Omega(1)$ and $1-\Omega(1)$. The expected time to find the whole Pareto front is $n^{\Omega(n)}$.
\end{theorem}

The remainder of this subsection is devoted to the proof of Theorem~\ref{the:probability-of-gaps}. First, we define what a gap means and transition probabilities for mutations on the Pareto front that will be used in the remainder of this section. 

\begin{definition}[Gap]
\label{def:gap}
We say that a population $P_t$ has a gap at position~$i$ if $1^i 0^{n-i} \notin P_t$, but $1^j 0^{n-j} \in P_t$ and $1^k 0^{n-k} \in P_t$ for $j < i < k$.
\end{definition}

\begin{definition}[Transition probabilities]
\label{def:tranpro}
We define 
\[
p_k = n^{-k} \cdot \left(1-\frac{1}{n}\right)^{n-k} = \left(1-\frac{1}{n}\right)^n \cdot (n-1)^{-k}.
\]
as the probability of jumping from any search point $1^i 0^{n-i}$ to $1^{i+k}0^{n-i-k}$  and $1^{i-k} 0^{n-i+k}$ (if existent).
\end{definition}

Next, we show that, once the Pareto front has been reached, the Highest Diversity Contribution selection will always choose a parent $x$ with an extreme number of ones. The following lemma applies to a population~$P$ containing only search points on the Pareto front. This setting applies for the modified GSEMO once the Pareto front has been reached as then parent selection is only based on search points with a maximum $\ld$-dominant attribute, corresponding to points on the Pareto front.

\begin{lemma}
\label{lem:hdcsel}
Consider the Highest Diversity Contribution (HDC) selection mechanism, choosing as diversity metric either $\HVC{x}{P}$ with reference point dominated by $(-n^2, -n^2)$ or $\CDC{x}{P}$ on the function \lotz, for a population~$P$ containing only search points on the Pareto front. Then the parent chosen by HDC will always either have a minimum or a maximum number of ones among all search points in~$P$.
\end{lemma}
\begin{proof}
Let us consider an individual $x_i$ of the sorted population according to $f_1$, using the notation from Definition~\ref{def:hypcon}, and let us define $\f{1}{x_0}\le-n^2$ and $\f{2}{x_{\mu+1}}\le-n^2$ as reference point. For any point $x_i=1^j0^{n-j}$ where $1 < i < \mu$, the highest possible contribution that the point $x_i$ can achieve is if it has as neighbours the points $x_1=0^n$ and $x_\mu=1^n$, so we have $\f{1}{x_i}=j$, $\f{1}{x_{i-1}}=\f{1}{x_1}=0$ and $\f{2}{x_i}=n-j$, $\f{2}{x_{i+1}}=\f{2}{x_\mu}=0$. In this sense, by Definition~\ref{def:hypcon}, the highest possible contribution for $x_i$ is $\HVC{x_i}{P}\le (j-0)\cdot(n-j-0)\le j\cdot(n-j)$ and since $j$ is restricted to $0<j<n$, the maximum contribution possible for $x_i$ is when $j=n/2$ and $n-j=n/2$ achieving $\HVC{x_i}{P}\leq n^2/4<n^2$.

For the case of points $x_1=0^n$ or $x_\mu=1^n$, the hypervolume contribution of any of the these two points is at least $n^2$, since the lowest possible contribution for these points is obtained when the individuals $x_2$ or $x_{\mu-1}$ are contained in the population, then $\HVC{x_1}{P}\le n^2\cdot 1$ (the same for $x_\mu$). So we have that $\HVC{x_1}{P}>\HVC{x_i}{P}$ and $\HVC{x_\mu}{P}>\HVC{x_i}{P}$ for all $1 < i < \mu$. 

For the case of CDC, both extreme points are always assigned an infinite diversity score, the highest possible score given by the CDC metric. So all intermediate individuals are ignored by the selection mechanism and HDC only selects the individual with the highest number of zeroes or ones in the population. 
\end{proof}

We further show that gaps emerge and remain with constant probability.
\begin{lemma}
\label{lem:gaps}
In the setting of Theorem~\ref{the:probability-of-gaps}, with probability $\Omega(1)$ the modified GSEMO will evolve a population with a gap at position $n/4 \le i \le 3n/4$.

The probability that this gap will remain at the first generation where the population contains both $0^n$ and $1^n$ is $\Omega(1)$.
\end{lemma}
\begin{proof}
In the following, we identify a search point $1^i 0^{n-i}$ with its index~$i$.
Note that, as long as no gap at index $n/4 \le i \le 3n/4$ is being created, the population spreads on this subset of the Pareto front as one Hamming path. This Hamming path is likely to start at some index $n/4 \le i \le 3n/4$ and then spread towards lower and higher indices, but it could also start at an index $i < n/4$ and spread towards higher indices, or start at $i > 3n/4$ and spread towards lower indices.
This means that, for every index $n/4+1 \le j \le 3n/4-1$ there will eventually be a search point $1^j 0^{n-j}$ that will be chosen as parent, and (depending on the direction of the spread) at least one Hamming neighbour from $\{1^{j-1} 0^{n-j+1}, 1^{j+1}0^{n-j-1}\}$ will not be contained in the population. Without loss of generality let this be $1^{j-1}0^{n-j+1}$ (the other case is symmetric). Then with probability at least $p_2$ a mutation of $1^j 0^{n-j}$ will create a search point with smaller index than $j-1$, creating a gap at position $j-1$. With probability $p_1$ the modified GSEMO will create $1^{j-1} 0^{n-j+1}$, and there will never be a gap at position $j-1$. Considering these two events, the conditional event of creating a gap, given that another search point on the front with smaller index is created, is at least 
\[
\frac{p_2}{p_1 + p_2} \ge \frac{p_2}{p_1} = \frac{1}{n-1}.
\]
The probability that at least one index $n/4 + 1 \le j \le 3n/4-1$ (of which there are $n/2 - O(1)$ many) will lead to the creation of a gap is at least
\begin{align*}
&1 - \left(1 - \frac{1}{n-1}\right)^{n/2-O(1)}\\
=\;& 1 - \left(1 - \frac{1}{n-1}\right)^{(n-1)/2} \cdot \left(1 - \frac{1}{n-1}\right)^{O(1)}
\ge 1 - e^{-1/2} - O(1/n)=\Omega(1)
\end{align*}
where the inequality used $\left(1 - \frac{1}{n-1}\right)^{n-1} \le 1/e$ and Bernoulli's inequality.

Now assume that a gap has been created at position~$g$ with $n/4 \le g \le 3n/4$. From here on, every index $1 \le i \le n-1$ has a chance to fill the gap if the population contains $1^i 0^{n-i}$, this search point is being chosen as parent, and mutation flips $|g-i|$ bits to create $1^g 0^{n-g}$, hence filling the gap. Note that, if $1^i 0^{n-i}$ is picked as parent, and without loss of generality $i < g$, if mutation creates an offspring $1^j0^{n-j}$ with $j < i$ then $1^i 0^{n-i}$ will never be selected as parent again, and the gap at~$g$ will never be filled from index~$i$. Considering these two events, the conditional probability of \emph{not} filling the gap from index $i < g$ is at least
\[
\frac{p_{1}}{p_1 + p_{g-i}} = \frac{(n-1)^{-1}}{(n-1)^{-1} + (n-1)^{i-g}} = \frac{1}{1+(n-1)^{1+(i-g)}}.
\]
The above is $1/2$ if $i=g-1$ and at least $1-(n-1)^{1-|i-g|}$ for $i < g-1$. The same probability bounds hold for $i=g+1$ and $i>g+1$, respectively.
Note that mutations from index $i$ are independent from mutations on other indices, hence we can multiply probability bounds for all indices $i \neq g$. 
Hence, the probability that the gap is \emph{not} filled from any index $i \neq g$ is at least
\begin{align*}
& \frac{1}{2} \cdot \prod_{1 \le i < g-1} \left(1 - (n-1)^{1-|i-g|}\right)
\cdot
\frac{1}{2} \cdot \prod_{g+1 < i \le n-1} \left(1 - (n-1)^{1-|i-g|}\right)\\
\ge\;& \frac{1}{4} \cdot \left(\prod_{d=2}^\infty \left(1 - (n-1)^{1-d}\right)\right)^2\\
\ge\;& \frac{1}{4} \cdot \left(1 - \sum_{d=2}^\infty (n-1)^{1-d}\right)^2\\
=\;& \frac{1}{4} \cdot \left(1 - \sum_{d=1}^\infty (n-1)^{-d}\right)^2\\
%=\;& \frac{1}{4} \cdot \left(1 - \frac{1/(n-1)}{1-1/(n-1)}\right)^2\\
=\;& \frac{1}{4} \cdot \left(1 - \frac{1}{n-2}\right)^2 = \Omega(1).\qedhere
%\ge\;& \frac{1}{4} - \frac{1}{2n-4}.
\end{align*}
\end{proof}

Now we can make use of Lemma~\ref{lem:gaps} to prove Theorem~\ref{the:probability-of-gaps}.

\begin{proof}[Proof of Theorem~\ref{the:probability-of-gaps}]
A sufficient condition for finding all points on the Pareto front in the setting of Theorem~\ref{the:probability-of-gaps} is to always create a new point on the Pareto front via 1-bit mutations. Because global mutations are used, it is possible to create a new search point on the Pareto front by making a $k$-bit jump, for $k \ge 2$, with probability $p_k$.

Let $E$ be the event that a new point is created on the Pareto front via 1\nobreakdash-bit flip, and let $B$ be the event of creating a new point on the Pareto front. We have $\Prob{E} \ge p_1$, where the inequality becomes an equality if there is only one possible 1-bit flip applicable. The probability of event $B$ is at most $\Prob{B} \le \Prob{E} + 2p_2 + 2p_3 + \dots + 2p_n$, taking into account all possible jump lengths, and the fact that it may be possible to make jumps in both directions. The conditional probability of event $E$ is at least

\begin{align*}
\Prob{E\mid B}\geq\;& \frac{p_1}{p_1+2p_2+2p_3+\ldots+2p_n} \\
=\;& \frac{\left(1-\frac{1}{n}\right)^n\cdot(n-1)^{-1}}{\left(1-\frac{1}{n}\right)^n\cdot((n-1)^{-1}+2(n-1)^{-2}+\ldots+2(n-1)^{-n})} \\
=\;&\frac{1}{1+2(n-1)^{-1}+\ldots+2(n-1)^{n-1}}\\
\geq\;&\frac{1}{1+2\displaystyle\sum_{i=1}^{\infty}(n-1)^{-i}}\\
=\;&\frac{1}{1+\frac{2}{n-2}} = 1 - \frac{\frac{2}{n-2}}{1+\frac{2}{n-2}} = 1-\frac{2}{n}.
\end{align*}

Now, the same probability bounds hold for all $i$ on the Pareto front. Mutations from point $i$ are independent from mutations on other indices, hence we can multiply the probability for all indices $i$. Hence, the probability of creating a new point due to 1-bit mutation is at least
\[
\prod_{i=1}^{n}\left(1-\frac{2}{n}\right)=\left(1-\frac{2}{n}\right)^n=\Omega(1).
\]

Now we have proved that the modified GSEMO is able to find all points on the Pareto front via 1-bit mutation, and by Lemma~\ref{lem:gaps}, the modified GSEMO will create a gap at position $n/4\leq i\leq 3n/4$ via more than 1-bit flip and this gap will remain after the points $0^n$ and $1^n$ have been found with probability $\Omega(1)$. At this point it will be necessary to flip at least $n/4$ specific number of bits from one of the extreme points in order to ``fill'' a gap. By Definition~\ref{def:tranpro}, the probability of making a $n/4$ jump from any extreme point is at most
\[
p_{n/4} = \left(1-\frac{1}{n}\right)^n \cdot (n-1)^{-n/4} = n^{-\Omega(n)}.
\]

Since the above probability bound holds for all current $n/4\leq i\leq 3n/4$ gaps, we get that the algorithm requires at least exponential runtime $n^{\Omega(n)}$ to fill all the remaining $i$ gaps.
\end{proof}

Note that the poor performance of the modified GSEMO is down to the choice of mutation operator, and the possibility of flipping multiple bits in one mutation. In contrast, SEMO using local mutations finds the Pareto front efficiently when HDC is used.
\begin{theorem}
Consider SEMO with Highest Diversity Contribution, choosing as diversity metric either CDC or HVC with a reference point $(-r, -r)$ for $r \ge 1$ on the function \lotz. Then the expected time for finding the whole Pareto front is $O(n^2)$.
\end{theorem}
\begin{proof}
We already know that SEMO reaches the Pareto front in expected time $O(n^2)$. Afterwards, the population spreads on the Pareto front as one Hamming path. Let $P = \{1^i0^{n-i}, 1^{i+1}0^{n-i-1}, \dots,\allowbreak 1^{j-1}0^{n-j+1} 1^j0^{n-j}\}$ be the current population sorted according to the number of ones, with $i, j$ being the minimum and maximum number of ones, respectively. 

Then for any $k$ with $i < k < j$ we have $\HVC{1^k0^{n-k}}{P} = 1$ in addition to $\HVC{1^i0^{n-i}}{P} = i+r$ and $\HVC{1^j0^{n-j}}{P} = n-j+r$. The latter two values simplify to~$r$ if $i=0$ or $j=n$, respectively, that is, for $0^n$ and $1^n$. For all values $i > 0$ we have $\HVC{1^i 0^{n-i}}{P} \ge r+1$ and the same holds for $j < n$ implying $\HVC{1^i0^{n-j}}{P} \ge r+1$. This implies that the highest diversity contribution is always attained for a good search point, as long as the whole Pareto front has not been found yet. In other words, $\pgood = 1$ and we obtain an upper bound of $O(n^2)$ by following the arguments from Section~\ref{sec:lotzpro}. 

For CDC, we have $\pgood \ge 1/2$ as both $1^i0^{n-i}$ and $1^j0^{n-j}$ have a crowding distance contribution of $\infty$, and at least one of them must be different from $0^n$ and $1^n$. The upper bound of $O(n^2)$ follows as before.
\end{proof}

\subsection{NMUAR is Fast but Brittle}
\label{sec:disnmuars}

As mentioned previously and based on the results of Table \ref{tab:expomm}, \ref{tab:explotz} and~\ref{tab:explranlotz}, NMUAR empirically performs well for SEMO, GSEMO and the modified GSEMO with any diversity metric in its different variants. No stagnation was detected during the experimental analysis made in Section~\ref{sec:exp}. It seems that the selection mechanism performs better compared with the other selection approaches. Nevertheless, as an observant reviewer for~\cite{Covantes2017} pointed out, it is possible to find populations where the probability of selecting a good parent is $0$, and the analytical framework used in Sections~\ref{sec:ommpro} and~\ref{sec:lotzpro} breaks down. 

Two such populations are shown in Figure~\ref{fig:counterex}. In Figure~\ref{fig:n}, $1^n$ is \bad as it can not produce yet unseen points with local mutations on the Pareto front. However, depending on the choice of reference point, it may have the highest hypervolume contribution. The remaining point, while being \good has the minimum hypervolume contribution. So, it is never picked as a parent by the NMUAR scheme. This means that the algorithm will never select a good search point, which leads to a stagnation state. Furthermore, Figure~\ref{fig:somen} shows that for the case of certain problem sizes, more points can be added on the Pareto front, such that all non-boundary points feature the same, minimum hypervolume contribution, and all of these points are ignored by NMUAR, leaving only bad search points $0^n$ and $1^n$ that may be selected as parents. This also shows that the example from Figure~\ref{fig:n} is not unique. 

\begin{figure*}[ht]
    \centering
    \begin{subfigure}[t]{0.49\textwidth}
        \centering
        \resizebox{\linewidth}{!}{
        	\begin{tikzpicture}[domain=0:8,xscale=0.8,yscale=0.8,scale=1, every shadow/.style={shadow xshift=0.0mm, shadow yshift=0.4mm}]
          \tikzstyle{helpline}=[black,very thick];
          \tikzstyle{function}=[blue,thick];
          \tikzstyle{individual}=[blue,very thick];
          \tikzstyle{pareto}=[red,very thick];
		  \draw[black!40,line width=0.2pt,xstep=1,ystep=1] (0,0) grid (8.5,8.5);          
          \draw[helpline, -triangle 45] (0,0) -- (0,9) node[above] {\Large $f_{2}$}; 
          \draw[helpline, -triangle 45] (0,0) -- (9,0) node[right] {\Large $f_{1}$};
          
          \draw[helpline] (0,0) -- (-0.2,0) node[left] {\Large $0$};
          \draw[helpline] (0,4) -- (-0.2,4) node[left] {\Large $n/2$};
          \draw[helpline] (0,8) -- (-0.2,8) node[left] {\Large $n$};
          \draw[helpline] (0,0) -- (0,-0.2) node[below] {\Large $0$};
          \draw[helpline] (4,0) -- (4,-0.2) node[below] {\Large $n/2$};
          \draw[helpline] (8,0) -- (8,-0.2) node[below] {\Large $n$};
          \draw[function] (0,8) -- (8,0);
%		  \foreach \x/\y in {0/8,1/7,7/1,8/0}
		  \foreach \x/\y in {0/8,1/7}
                \filldraw[pareto] (\x,\y) circle (4pt);              
          \draw[helpline] (0.3,7.7) -- (-0.2,7.55) node[left] {\Large $F^*_n$};      
		\end{tikzpicture}
		}
        \caption{$n=8$}
        \label{fig:n}
    \end{subfigure}
    \begin{subfigure}[t]{0.49\textwidth}
        \centering
        \resizebox{\linewidth}{!}{
        	\begin{tikzpicture}[domain=0:7,xscale=0.7,yscale=0.7,scale=1, every shadow/.style={shadow xshift=0.0mm, shadow yshift=0.4mm}]
          \tikzstyle{helpline}=[black,very thick];
          \tikzstyle{function}=[blue,thick];
          \tikzstyle{individual}=[blue,very thick];
          \tikzstyle{pareto}=[red,very thick];
		  \draw[black!40,line width=0.2pt,xstep=1,ystep=1] (0,0) grid (7.5,7.5);          
          \draw[helpline, -triangle 45] (0,0) -- (0,8) node[above] {\Large $f_{2}$}; 
          \draw[helpline, -triangle 45] (0,0) -- (8,0) node[right] {\Large $f_{1}$};
          
          \draw[helpline] (0,0) -- (-0.2,0) node[left] {\Large $0$};
          \draw[helpline] (0,4) -- (-0.2,4) node[left] {\Large $n/2$};
          \draw[helpline] (0,7) -- (-0.2,7) node[left] {\Large $n$};
          \draw[helpline] (0,0) -- (0,-0.2) node[below] {\Large $0$};
          \draw[helpline] (4,0) -- (4,-0.2) node[below] {\Large $n/2$};
          \draw[helpline] (7,0) -- (7,-0.2) node[below] {\Large $n$};
          \draw[function] (0,7) -- (7,0);
		  \foreach \x/\y in {0/7,1/6,3/4,4/3,6/1,7/0}
                \filldraw[pareto] (\x,\y) circle (4pt);              
          \draw[helpline] (0.3,6.7) -- (-0.2,6.55) node[left] {\Large $F^*_n$};      
		\end{tikzpicture}
		}
        \caption{$n=7$}
        \label{fig:somen}
    \end{subfigure}
    \caption{Examples of populations where NMUAR with CDC or HVC may only select \bad individuals from $\{0^n, 1^n\}$ on \omm and/or \lotz, depending on the choice of reference point (all non-extreme points have the same score, NMUAR only selects extreme points).}
    \label{fig:counterex}
\end{figure*}
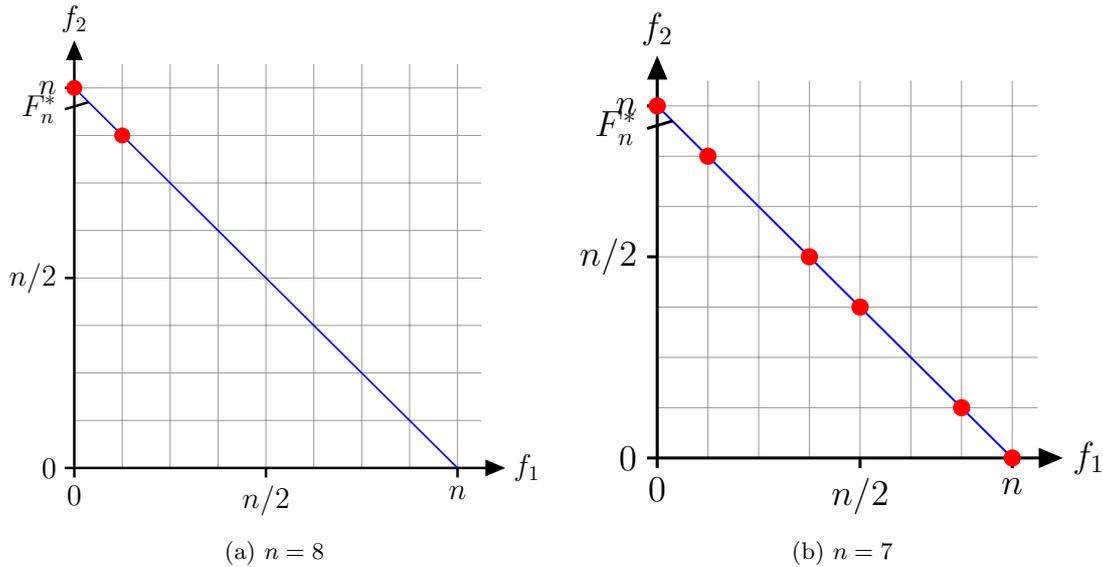

In the following we show that, despite these risks, NMUAR is able to efficiently optimise both \omm and \lotz. First we define the following probability of selecting \good individuals by providing some additional arguments on how to deal with different situations for $\pgood$.

\begin{lemma}
\label{lem:nmuar_prob}

Let $P$ denote the current population and $P' \subseteq P$ denote the population from which NMUAR selects uniformly at random. Consider \omm or \lotz and assume that the Pareto front has been reached, but $P$ does not cover the whole front.
The probability $\pgood$ of selecting a \good individual using CDC or HVC with any reference point dominated by $(-1,-1)$ and NMUAR selection is $\pgood=\Omega(1)$ if one of the following conditions is met:

\begin{enumerate}
\item $P$ contains neither $0^n$ nor $1^n$,
\item $P$ contains a search point $x \in \{0^n, 1^n\}$ and $x$ is good, or
\item $P$ contains individuals with $f_1$-values $i$, $i+1$, and~$i+2$, for some value $0 \le i \le n-2$.
\end{enumerate}
\end{lemma}
\begin{proof}
For HVC, note that any potential bad search points from $\{0, 1\}^n \setminus \{0^n, 1^n\}$ will
have the same diversity score of 1, which is minimal amongst all possible HVC values. All good search points have a larger diversity score. This means that NMUAR will never choose a parent with minimum HVC score. The same applies to CDC where the minimum value depends on values $f_m^{\min}$ and $f_m^{\max}$.

The only risk is that NMUAR may choose a bad search point from $\{0^n, 1^n\}$. While the population does not contain any such points, $\pgood=1$. As long as the population contains a search point $x \in \{0^n, 1^n\}$ and $x$ is good, we have $\pgood \ge 1/2$ as there can only be one potential bad search point, namely $\overline{x}$, that has a chance to be selected.

The third condition implies that the individual with $f_1$-value $i+1$, which is a bad individual, has a minimum diversity contribution. Hence NMUAR will only remove bad individuals and all good individuals will remain. 
There will be at least one good search point $x \in P$ as long as the population does not cover the whole front; it can be found by scanning the Pareto front, starting at $i$ and moving towards smaller $f_1$-values and starting at $i+2$ in the direction of larger $f_1$ values. In both directions either a search point from $\{0^n, 1^n\}$ or a good search point will be found. 
As the whole front has not been covered yet, at least one direction will result in a good search point. As there can be at most two bad search points in $P'$ ($0^n$ and $1^n$), $\pgood \ge 1/3$.
\end{proof}

Now we can prove the following theorem.

\begin{theorem}
\label{the:nmuar_omm}
The expected time for SEMO and GSEMO to find the whole Pareto front on \omm is bounded by $O(n \log n)$ for the NMUAR selection scheme with either CDC or HVC with a reference point $(-r, -r)$ for $r \ge 1$.
\end{theorem}
\begin{proof}
%Lemma~\ref{lem:nmuar_prob} gives sufficient conditions for $\pgood = \Omega(1)$. 
Let $P$ and $P'$ be as in Lemma~\ref{lem:nmuar_prob}. Whenever $\pgood = \Omega(1)$ we can apply the arguments from Section~\ref{sec:ommpro}, but we need to provide additional arguments to deal with possible settings where $\pgood$ is not guaranteed to be~$\Omega(1)$.
In order for $\pgood \notin \Omega(1)$ to hold, we must have $P' \subseteq \{0^n, 1^n\}$ with all members of $P'$ being bad. This implies that, if $1^n \in P'$, the population must contain a search point with $n-1$ ones (as otherwise $1^n$ would be good) and it cannot contain any search point with $n-2$ ones (as otherwise the third condition of Lemma~\ref{lem:nmuar_prob} would be true). The same logic applies to $0^n$ and its neighbours.

We show that such a pathological case where $\pgood \notin \Omega(1)$ is impossible for SEMO, due to our assumptions on the choice of reference point $(-r, -r)$. For \omm all points are in the Pareto front, and because local mutations are being used, the population always contains all possible $f_1$ values in some integer range. 
Hence the population can only be $P = \{1^n, x\}$ where $x$ has $n-1$ ones, or $P = \{0^n, x'\}$ where $x'$ has a single one. W.\,l.\,o.\,g.\ the former is the case. Then CDC assigns value $\infty$ to both search points, hence $\pgood = 1/2$. For HVC we have $\hvc(1^n) = r$ and $\hvc(x) = n-1+r$, hence $P' = \{x\}$ and $\pgood = 1$.

For GSEMO, if $P' \subseteq \{0^n, 1^n\}$ and w.\,l.\,o.\,g.\ $1^n \in P'$, $1^n$ is selected as parent with probability at least $1/2$. Any mutation of $1^n$ flipping two arbitrary bits will create a search point with $n-2$ ones, which then fulfils the third condition from Lemma~\ref{lem:nmuar_prob} for the next and all future populations. The expected waiting time for making this mutation is $O(1)$.
Afterwards, $\pgood = \Omega(1)$ by Lemma~\ref{lem:nmuar_prob} and we obtain an upper bound for both SEMO and GSEMO of $O(n \log n)$ following the previous analyses from Section~\ref{sec:ommpro}. 
\end{proof}

Similar arguments can be used to prove that SEMO and the modified GSEMO can optimise \lotz efficiently.
\begin{theorem}
\label{the:nmuar_lotz}
The expected time for SEMO and the modified GSEMO to find the whole Pareto front on \lotz is bounded by $O(n^2)$ for the NMUAR selection scheme with either CDC or HVC with a reference point $(-r, -r)$ for $r \ge 1$.
\end{theorem}
\begin{proof}
By the same arguments as in the proof of Theorem~\ref{the:nmuar_omm}, in order to have $\pgood \notin \Omega(1)$ the population must contain $1^n$ and $1^{n-1}0$, but not $1^{n-2}00$, or the symmetric constellation involving $0^n$, $10^{n-1}$, and $110^{n-2}$. For SEMO, arguing as in the proof of Theorem~\ref{the:nmuar_omm} the choice of reference point then implies that $\hvc(1^{n-1}0) > \hvc(1^n)$, hence we must always have $\pgood = \Omega(1)$. 

For the modified GSEMO, if $P' \subseteq \{0^n, 1^n\}$ and w.\,l.\,o.\,g.\ $1^n \in P'$, $1^n$ is selected as parent with probability at least $1/2$. The probability of a mutation turning $1^n$ into $1^{n-2}00$ is at least $1/(en^2)$, and once it occurs, it fulfils the third condition from Lemma~\ref{lem:nmuar_prob} for the next and all future populations. The expected waiting time for making this mutation is $O(n^2)$. 
Afterwards, $\pgood = \Omega(1)$ by Lemma~\ref{lem:nmuar_prob} and we obtain an upper bound for both SEMO and the modified GSEMO of $O(n^2)$ following the previous analyses from Section~\ref{sec:lotzpro}. 
\end{proof}

Note that NMUAR is not robust to the choice of the reference point. The proof of Lemma~\ref{lem:nmuar_prob} has revealed a scenario where, with an asymmetric choice of the reference point, SEMO can get stuck.
\begin{theorem}
There is a choice of reference point in the area dominated by $(-1, -1)$ such that SEMO with HVC and NMUAR selection has a positive probability of stagnating on \omm and \lotz.
\end{theorem}
\begin{proof}
Choose the reference point as $(-n-1, -1)$. With positive probability, SEMO is initialised with $1^n$. Then only offspring with an $f_1$ value of $n-1$ are accepted. Once the population equals $P=\{1^n, x\}$, where $f_1(x) = n-1$, we have $\hvc(1^n) = n+1$ and $\hvc(x) = n$, hence NMUAR will always choose $1^n$ as parent, leading to stagnation. 
\end{proof}

\section{Discussion and Conclusions}
\label{sec:con}

Diversity plays a crucial role in the area of EMO. So far, diversity-based parent selection has not been the main focus on algorithm design. We have proposed a range of diversity-based parent selection schemes, aiming to speed up the spread on the Pareto front. We have demonstrated for two example functions, \omm and \lotz, that our new selection schemes can significantly speed up EMO algorithms. Our theoretical results show that a linear factor can be saved for the investigated settings and this is confirmed by our experimental results showing a speedup of one magnitude for problems of size $n=100$.

We have analysed different selection schemes with different preference toward the individual's diversity contribution, from aggressive schemes that put a strong emphasis on individuals with the highest diversity contribution to more relaxed schemes that introduce a bias for more diversity, but still give all individuals a chance to be selected as parents.

The analysis has shown that very extreme schemes can lead to undesired results. For  selection mechanisms that entail a rather extreme change of behaviour, such as Highest Diversity Contribution (HDC) and Non-Minimum Uniformly at Random (NMUAR), search may stagnate. On the other hand, our rank-based approaches as well as tournament selection are successful for \omm and \lotz, for both SEMO and GSEMO. Among these, the power law selection scheme is the fastest, hence we recommend this scheme as having the best trade-off between speed and risk. We believe the power law selection to also be beneficial for other problems as it has a high probability of selecting parents with the highest diversity contribution, but it also has a fat tail, allowing any individual to still be selected as parent with a reasonable probability.

Our theoretical analysis of stagnation behaviour has further revealed an interesting and quite natural setting where standard bit mutations are detrimental in MOEAs, compared to local mutations flipping only one bit. The performance difference is very drastic as the choice of the mutation operator decides between an expected polynomial time and exponential time for finding the whole Pareto front.

For future work, it would be interesting to study the benefit of diversity-based parent selection on more complex problems. From a theoretical perspective, combinatorial optimisation problems such as minimum spanning tress and covering problems for which SEMO has already been studied would be natural candidates. On the experimental side, it would be interesting to integrate the presented diversity-based parent selection methods into state-of-the-art EMO algorithms and to evaluate their performance on well-established benchmark sets.

\section*{Acknowledgements}
The authors would like to thank the anonymous reviewers of the previous GECCO publication for their many valuable suggestions which improved the paper, especially the anonymous reviewer who provided some comments and examples where NMUAR could fail. This research has been supported by the Consejo Nacional de Ciencia y Tecnolog\'{i}a --- CONACYT (the Mexican National Council for Science and Technology) under the grant no. 409151 and registration no. 264342, and by Australian Research Council (ARC) grants DP140103400 and DP160102401. The research leading to these results has received funding from the European Union Seventh Framework Programme (FP7/2007-2013) under grant agreement no.\ 618091 (SAGE).

% To change the title from References to Bibliography:
%\renewcommand\refname{Bibliography}

\bibliographystyle{abbrvnat} % or try abbrvnat or unsrtnat
\bibliography{mybibfile} % refers to example.bib

\end{document}